\documentclass[letterpaper,11pt]{article}

\usepackage[utf8]{inputenc} 
\usepackage[T1]{fontenc}    
\usepackage{lmodern}
\usepackage[colorlinks]{hyperref}       
\usepackage{url}            
\usepackage{booktabs}       
\usepackage{amsfonts}       
\usepackage{nicefrac}       
\usepackage{microtype}      
\usepackage[title]{appendix}
\usepackage{enumitem}
\usepackage{authblk}

\newif\ifdraft
\drafttrue

\usepackage{algorithm}
\usepackage{algorithmic}
\usepackage{amssymb}
\usepackage{graphicx}
\usepackage{fullpage}
\usepackage{comment}
\usepackage[colorlinks]{hyperref}
\usepackage{amsthm,amsmath}
\usepackage[numbers]{natbib}
\usepackage{mathtools}
\newtheorem{theorem}{Theorem}

\newtheorem{definition}[theorem]{Definition}

\newtheorem{corollary}[theorem]{Corollary}

\newtheorem{lemma}[theorem]{Lemma}

\newtheorem{prop}{Proposition}
\newtheorem{assume}{Assumption}
\allowdisplaybreaks

\title{Collaborative Multi-Agent Heterogeneous Multi-Armed Bandits}
\author[1]{Ronshee Chawla\thanks{Corresponding author, email: ronsheechawla@utexas.edu\vspace{-2em}}}
\author[1,2]{Daniel Vial}
\author[1]{Sanjay Shakkottai}
\author[2]{R. Srikant}
\affil[1]{University of Texas at Austin}
\affil[2]{University of Illinois at Urbana-Champaign\vspace{-3em}}
\date{}

\begin{document}
\maketitle
\begin{abstract}
The study of collaborative multi-agent bandits has attracted significant attention recently.
In light of this, we initiate the study of a new collaborative setting, consisting of $N$ agents such that each agent is learning one of $M$ stochastic multi-armed bandits to minimize their group cumulative regret.
We develop decentralized algorithms which facilitate collaboration between the agents under two scenarios. 
We characterize the performance of these algorithms by deriving the per agent cumulative regret and group regret upper bounds.
We also prove lower bounds for the group regret in this setting, which demonstrates the near-optimal behavior of the proposed algorithms.
\end{abstract}

\section{Introduction}
The multi-armed bandit (MAB) problem is a paradigm for sequential decision-making under uncertainty, which involves allocating a resource to an action, in order to obtain a reward.
MABs address the tradeoff between exploration and exploitation while making sequential decisions.
Owing to their utility in large-scale distributed systems (such as information retrieval \cite{yue2009interactively}, advertising \cite{chakrabarti2008mortal}, etc.), an extensive study has been conducted on multi-agent versions of the classical MAB in the last few years.
In multi-agent MABs, there are multiple agents learning a bandit and communicating over a network.
The goal is to design communication strategies which allow efficient exploration of arms across agents, so that they can perform better than single agent MAB algorithms.\\

There exist many versions of multi-agent MABs in the literature (see Section \ref{subsec:relatedwork} for an overview). 
We propose a new collaborative setting where each of the $N$ agents is learning one of $M$ stochastic MABs (where each of the bandits have $K$ arms and $M < N$) to minimize the group cumulative regret, i.e., the sum of individual cumulative regrets of all the agents.
We assume that every bandit $m \in [M]$ is learnt by $N_{m} = \Theta(\frac{N}{M})$ agents. 
At time $t \in \mathbb{N}$, for all $m \in [M]$, any agent learning the $m^{\mathrm{th}}$ bandit plays one of the $K$ arms and receives a stochastic reward, independent of everything else (including the other agents learning the $m^{\mathrm{th}}$ bandit pulling the same arm).
The network among the agents is denoted by a $N \times N$ gossip matrix $G$ (fixed and unknown to the agents), where every $n^{\mathrm{th}}$ row of $G$ ($n \in [N]$) is a probability mass function over the set $[N]$.
We investigate our setting under two scenarios: (a) \emph{context unaware}, in which no agent is aware of the other agents learning the same bandit, and (b) \emph{partially context aware}, in which every agent is aware of $r-1$ other agents learning the same bandit.
In the \emph{context unaware} scenario, in addition to the arm pulls conducted at each time, each agent can choose to pull information from another agent who is chosen at random based on the gossip matrix $G$.
In the \emph{partially context aware} scenario, every agent can also choose to exchange messages with the $r-1$ agents that they know are learning the same bandit, in addition to the arm pulls and the information pulls.
Agents behave in a decentralized manner, i.e., the arm pulls, information pulls and messages exchanged are only dependent on their past actions, obtained rewards and the messages received from other agents.\\


{\bf Problem Motivation:} Our setting finds its utility in applications involving multiple agents, with possibly differing contexts.
While agents with the same context have the same objective, they may not be aware of who else shares the same context. 
Analyzing the proposed setting is an attempt in figuring out the best way agents should utilize recommendations from others.
The quality of the recommendations received from other agents depends on whether those agents share the same context (and hence have the same objective).

For example, consider a social recommendation engine like Yelp, where the $N$ agents correspond to the users, each choosing one from among $K$ pizzerias (arms) in a particular location; this can be modeled as a $K$-armed MAB.
Suppose that each pizzeria serves multiple types of pizza such as Neapolitan, Chicago deep dish, New York style, Detroit style, etc., where each pizza style corresponds to a bandit. 
Each user has a preference for one of the pizza types and is looking for the corresponding best pizzeria. For instance, pizzeria 1 (arm 1) might excel in Chicago deep dish, while pizzeria 2 might be best for Detroit style. If a user (agent) does not know who else has similar tastes (i.e., has the same context), the user would not be able to determine whether a review recommendation is ``useful'', especially if the review only states that a pizzeria is excellent, but does not describe the type of pizza that the reviewer consumed. What we have here is a setting where an arm (pizzeria) is being recommended by some user, but for a specific user who is perusing the reviews, it is not clear if that recommendation is actually useful (e.g., the recommendation is actually for a Detroit style pizza, but the user is interested in Chicago deep dish).
%
This example corresponds to the \emph{context unaware} scenario in our model.
Moreover, if a user knows another user or a group of users who share the same pizza type preference, they can exchange pizzeria recommendations while still scrolling through the reviews of all the pizzerias by themselves, which corresponds to the \emph{partially context aware} scenario in our model.


\subsection{Key Contributions}
{\bf Algorithms:} For the \emph{context unaware} scenario, we modify and subsequently analyze the GosInE algorithm (Algorithm \ref{alg:gosine}), which \cite{aistats2020gossip} proposed for the case of a single bandit ($M=1$).  For the \emph{partially context aware} scenario, we utilize the insights obtained from the analysis of Algorithm \ref{alg:gosine} and propose a new algorithm (Algorithm \ref{alg:gosinesideinfo}).
Algorithm \ref{alg:gosine} proceeds in phases, such that agents play from a subset of the $K$ arms within a phase, and recommend the arm ID of the best arm during information pulls at the end of a phase.
The received arm recommendations are used to update the active sets before the next phase begins.
In the \emph{partially context aware} scenario, the agents additionally (a) share the arm recommendation with the $r-1$ agents known to be learning the same bandit, (b) determine the $M$ most recent unique arm recommendations among their set of $r$ agents, and (c) distribute these recent recommendations among themselves to update their respective active sets.\\

{\bf Gossip despite multiple bandits:} Our main analytical contribution is to show that under a mild assumption (Assumption \ref{assume:stickyset}), agents running Algorithm \ref{alg:gosine} for the \emph{context unaware} scenario and running Algorithm \ref{alg:gosinesideinfo} for the \emph{partially context aware} scenario are able to identify the best arm of the bandit they are learning, and are able to spread it quickly to the other agents learning the same bandit.
Even though the outcome is the same as the setting in which agents are collaboratively learning a single bandit \cite{aistats2020gossip, newton2022asymptotic}, the spreading of the best arms for each of the $M$ bandits is extremely complicated because the $M$ spreading processes are intertwined and evolving simultaneously (since every agent interacts with the agents learning other bandits). 
Consequently, agents cannot trust the arm recommendations received from the agents learning other bandits.
Thus, unlike \cite{aistats2020gossip, newton2022asymptotic}, the spreading process of each of the $M$ arms isn't bounded by a standard rumor spreading process, hence requiring an involved analysis (detailed in Section \ref{pfsk:gosine}).\\

{\bf Upper bounds:} We show that the expected cumulative regret of an individual agent running Algorithms \ref{alg:gosine} and \ref{alg:gosinesideinfo} scales as $O\left(\frac{\frac{MK}{N}+M}{\Delta_m}\log T\right)$ (Theorem \ref{thm:gosinereg} and Corollary \ref{cor:gosinereg}) and $O\left(\frac{\frac{MK}{N}+\frac{M}{r}}{\Delta_m}\log T\right)$ (Theorem \ref{thm:gosinesideinforeg} and Corollary \ref{cor:gosinesideinforeg}) respectively, for large $T$, where $\Delta_{m}$ is the minimum arm gap of the $m^{\mathrm{th}}$ bandit.
It is evident that when $M << \min \{K, N\}$, agents learning the $m^{\mathrm{th}}$ bandit for all $m \in [M]$ experience regret reduction compared to the case of no collaborations, because of the distributed exploration of the $K$ arms among $N_m = \Theta(\frac{N}{M})$ agents.
Furthermore, the expected group regret incurred by agents running Algorithms \ref{alg:gosine} and \ref{alg:gosinesideinfo} scales as $O\left(\sum_{m \in [M]}\frac{K+N}{\Delta_m}\log T\right)$ (Corollary \ref{cor:gosinegpreg}) and $O\left(\sum_{m \in [M]}\frac{K+\frac{N}{r}}{\Delta_m}\log T\right)$ (Corollary \ref{cor:gosinesideinfogpreg}) respectively.\\

{\bf Lower bounds:} 
We show in Theorem \ref{thm:lowerbound} that the expected group regret of our model scales as $\Omega\left(\sum_{m \in [M]}\sum_{k \neq k^{*}(m)}\frac{1}{\Delta_{m, k}}\log T\right)$ for large $T$, where $k^{*}(m)$ is the best arm and $\Delta_{m, k}$ is the arm gap of the $k^{\mathrm{th}}$ arm of the $m^{\mathrm{th}}$ bandit. This demonstrates that the first terms in our group regret upper bounds are near-optimal; we also show the second terms (which scale as $N \log T$) are unavoidable in general. See Section \ref{sec:lowerbound} for details.

\subsection{Related Work}
\label{subsec:relatedwork}
Our work falls broadly under the category of cooperative MABs.
To the best of our knowledge, the setting considered in this work hasn't been studied previously.
Our work is closest to the line of work in \cite{sankararaman2019social, aistats2020gossip, newton2022asymptotic, vial2021robust, vial2022robust}, which involves multiple agents learning the same $K$-armed MAB and communicating sub-linear number of times (during the entire time horizon) through bit-limited pairwise gossip style communications to minimize individual cumulative regret.
\cite{vial2021robust, vial2022robust} considers a multi-agent system with honest and malicious agents, where honest agents are learning the same $K$-armed MAB.
Their algorithms can be used in our setting, where an agent learning some bandit treats agents learning other bandits as malicious.
In our setting, an agent running those algorithms incur regret scaling as $O\left(\frac{1}{\Delta_m}\left(\frac{MK}{N} + N\left(1-\frac{1}{M}\right)\right) \log T\right)$ after $T$ time steps.
This regret scaling is problematic when $\frac{K}{N} = \Theta(1)$, as there is no benefit of collaboration: the regret scales as $O\left(\frac{K}{\Delta_m} \log T\right)$, which is akin to learning a $K$-armed MAB without communications.
In our work, we show that an agent using the GosInE Algorithm in \cite{aistats2020gossip} with a slight modification results in lesser regret in the \emph{context unaware} scenario and is further reduced in the \emph{partially context aware} scenario. 
Due to space constraints, we refer the reader to Appendix \ref{appdx:relatedwork} for other related work.

\section{Problem Setup}
\label{sec:setup}
We consider a multi-agent multi-armed bandit (MAB) setting consisting of $N$ agents, where each agent attempts to learn one of $M$ stochastic MABs (where $M < N$), each containing $K$ arms.
Formally, for each $m \in [M]$, let $\mathcal{I}_{m} \subset [N]$ denote the set of agents learning the $m^{\mathrm{th}}$ bandit.
We assume that every bandit $m \in [M]$ is learnt by $N_{m}$ agents, such that $\sum_{m \in [M]} N_{m} = N$ and $c_1 \frac{N}{M} \leq N_{m} \leq c_2 \frac{N}{M}$ for all $m \in [M]$, where $\frac{M}{N} < c_1 \leq c_2 \leq \frac{M}{2}$ are known absolute constants.
For every bandit $m \in [M]$, the $K$ arms have unknown mean rewards, denoted by $\{\mu_{m, k}\}_{k \in [K]}$, where $\mu_{m, k} \in [0, 1]$ for all $k \in [K]$.
Let $k^{*}(m)$ denote the best arm for the $m^{\mathrm{th}}$ bandit, i.e., $k^{*}(m) = \arg \max_{k \in [K]} \mu_{m, k}$, and assume that $\mu_{m, k^{*}(m)} > \mu_{m, k}$ for all $k \neq k^{*}(m)$.
We define $\mathcal{B}$ to be the set of $M$ best arms, i.e., $\mathcal{B} = \{k^{*}(m)\}_{m \in [M]}$ and $\mathcal{B}^{(-m)} = \mathcal{B} \backslash \{k^{*}(m)\}$. 
Let $\Delta_{m, k} := \mu_{m, k^{*}(m)} - \mu_{m, k}$ denote the arm gaps for all $k \neq k^{*}(m)$ and $m \in [M]$.
The assumption on the arm means implies that $\Delta_{m, k} \in [0, 1]$ for all $k \neq k^{*}(m)$.
Let $\Delta_{m}$ denote the minimum arm gap of the $m^{\mathrm{th}}$ bandit, i.e., $\Delta_{m} = \min_{k \in [K] \backslash k^{*}(m)} \Delta_{m, k}$. \\

For all $m \in [M]$, each agent $i \in \mathcal{I}_{m}$ at time $t \in [T]$ pulls an arm $I_t^{(i)} \in [K]$ and receives a reward $X_t^{(i)}(I_t^{(i)})$, where $X_t^{(i)}(k) = \mu_{m,k} + \delta_{t}^{(i)}$ for each $k \in [K]$ and $\delta_{t}^{(i)}$ is $1$-subgaussian noise (independent of everything else). 
The network between the agents is represented by a $N \times N$ gossip matrix $G$ (fixed and unknown to the agents), where each row of the matrix $G(n, .)$ denotes a probability distribution for all $n \in [N]$. In this work, we consider that the agents are connected by a complete graph, i.e., for all $n \in [N]$, $G(n, i) = (N-1)^{-1}$ for all $i \neq n$. \\

The problem setup is investigated under two scenarios: 

(i) \emph {Context Unaware} - No agent $i \in [N]$ knows which other agents are learning the same bandit.

(ii) \emph{Partially Context Aware} - For all bandits $m \in [M]$, each agent $i$ learning the $m^{\mathrm{th}}$ bandit knows $r-1$ other agents (where $1 < r \leq \min_{m' \in [M]} N_{m'}$) who are also learning the $m^{\mathrm{th}}$ bandit, such that $N_m$ is an integral multiple of $r$ for all $m \in [M]$\footnote{since $N = \sum_{m \in [M]} N_m$, $N$ is an integral multiple of $r$\vspace{-2em}}.\\

In the \emph{context unaware} scenario, after pulling an arm, agents can choose to receive a message from another agent through an information pull.
In particular, if an agent $n \in [N]$ decides to pull information, it does so by contacting another agent $i \in [N]$ according to the probability distribution $G(n, .)$, independently of everything else. The agent $i$ who is contacted is then allowed to communicate $\lceil \log_{2}K \rceil$ number of bits during this information pull.
In the \emph{partially context aware} scenario, in addition to the information pulls allowed in the \emph{context unaware} scenario, each agent learning the $m^{\mathrm{th}}$ bandit can also exchange messages with the $r-1$ other agents who they know are also learning the $m^{\mathrm{th}}$ bandit.
As a result, agents in this scenario are allowed to communicate $r\lceil \log_{2}K \rceil$ number of bits during information pulls.\\

Agents operate in a decentralized fashion, i.e., all the decisions that an agent makes can solely depend on their past actions, rewards and the messages received from other agents during information pulls. 
Moreover, decisions made by agents during the information pulling slots (i.e., what to communicate if asked for information) are allowed to be dependent on the arm pulls in those slots. \\

Under both the scenarios, we would like to leverage collaboration between the agents to minimize the expected group cumulative regret, i.e., the sum of the individual cumulative regrets for all the agents.
Mathematically, let $I_{t}^{(i)}$ denote the arm pulled by agent $i \in [N]$ at time $t \in [T]$ and $c(i)$ denote the index of the (unknown) bandit that agent $i$ is trying to learn, i.e., if $i \in \mathcal{I}_m$, $c(i)=m$.
Then, the cumulative regret of an agent $i \in [N]$ after playing for $T$ time steps is denoted by $R_{T}^{(i)} := \sum_{t=1}^{T} (\mu_{c(i),k^{*}(c(i))} - \mu_{c(i), I_{t}^{(i)}})$ and the expected group cumulative regret is given by $\mathbb{E}[\mathrm{Reg}(T)]$\footnote{the expectation is with respect to the rewards, communications and the algorithm\vspace{-2em}}, where $\mathrm{Reg}(T)= \sum_{i=1}^{N} R_{T}^{(i)}$.

\section{\emph{Context Unaware} Algorithm}
\label{sec:gosine}

For the \emph{context unaware} scenario, we consider the GosInE Algorithm in \cite{aistats2020gossip} with a slight modification (Algorithm \ref{alg:gosine}).
Subsequently, we demonstrate that under a mild assumption (Assumption \ref{assume:stickyset}), agents running Algorithm \ref{alg:gosine} incur less regret compared to the case when they are learning their bandit without collaboration, despite being unaware of the other agents learning the same bandit.
Furthermore, we will show that this regret is near-optimal by stating a lower bound in Section \ref{sec:lowerbound}. 

\subsection{Key Algorithmic Principles}
The GosInE Algorithm in \cite{aistats2020gossip} has the following key components:\\

{\bf{Phases -}} The algorithm proceeds in phases $j \in \mathbb{N}$, such that during a phase, agents play from a set $S_{j}^{(i)} \subset [K]$, also known as active set. At the end of a phase, agents exchange arm recommendations through pairwise gossip communication and update their active sets.\\

{\bf{Active Sets -}} The active set for any agent $i \in [N]$ is a combination of two sets: (i) the time-invariant sticky set, denoted by $\widehat{\mathcal{S}}^{(i)} \subset [K]$, (ii) the non-sticky set. As the name suggests, the sticky set $\widehat{\mathcal{S}}^{(i)} \subset S_{j}^{(i)}$ for all $j \in \mathbb{N}$, i.e., it is always present in the active set. 
The non-sticky set contains two arms at all times, which are updated across the phases through arm recommendations received from other agents.\\

{\bf{Arm Recommendations -}} At the end of phase $j$, every agent contacts another agent at random based on the network's gossip matrix (which in this work is a complete graph), and the contacted agent sends the ``best'' arm in their active set. After receiving the recommendation, agents update their active set by adding the recommended arm to their sticky set and discarding the ``worst'' performing arm out of the two non-sticky arms. We provide an efficient modification to this update rule in Algorithm \ref{alg:gosine}, and provide the details about what it means to be the ``best'' and the ``worst'' performing arm in the next sub-section.


\subsection{Algorithm Description}
We provide a detailed description of the GosInE Algorithm with an efficient modification to the active set updates at the end of a phase, with the pseudo-code in Algorithm \ref{alg:gosine}.\\

{\bf{Initialization:}} The algorithm is initialized with the following inputs - (i) the standard exploration parameter of the UCB Algorithm, denoted by $\alpha > 0$,  (ii) the parameter $\beta > 1$ which controls the length of the phases, where a phase $j$ runs from the (discrete) time instants $1+A_{j-1}, \cdots, A_{j}$ with $A_{j} := \lceil j^{\beta} \rceil$, and (iii) a sticky set $\widehat{\mathcal{S}}^{(i)}$ of cardinality $S$ for each agent $i$. Note that the phases grow longer as the algorithm progresses (since $A_j - A_{j-1} = \Theta ( j^{\beta-1} )$ and $\beta > 1$). Details regarding the size of the sticky set are provided in the remarks at the end of this sub-section. For the first phase ($j=1$), we initialize the active set of each agent to be their sticky set, i.e., $S_{1}^{(i)} = \widehat{\mathcal{S}}^{(i)}$.\\

{\bf{UCB within a phase:}} We denote $T_{k}^{(i)}(t)$ to be the number of times agent $i$ has played an arm $k$ up to and including time $t$, and $\widehat{\mu}_{k}^{(i)}(t)$ to be the empirical mean reward among those plays, 
i.e., $\widehat{\mu}_{k}^{(i)}(t) = \frac{1}{T_{k}^{(i)}(t)} \sum_{s \leq t: I_{s}^{(i)} = k} X_{s}^{(i)}(I_s^{(i)})$ \footnote{$\mu_{k}^{(i)}(t) = 0$ if $T_{k}^{(i)}(t) = 0$\vspace{-2em}}. 
In phase $j$, every agent $i$ plays UCB Algorithm on their active set $S_{j}^{(i)}$, i.e., for $t \in \{1 + A_{j-1}, \cdots, A_{j}\}$, the chosen arm $I_{t}^{(i)} = \arg \max_{k \in S_{j}^{(i)}} \widehat{\mu}_{k}^{(i)}(t-1) + \sqrt{\frac{\alpha \log T}{T_{k}^{(i)}(t-1)}}$.\\

{\bf{Arm recommendation at the end of a phase:}} The arm recommendation received by agent $i$ when $t=A_{j}$ is denoted by $\mathcal{O}_{j}^{(i)} \in [K]$. During the information pull request at time $t=A_j$, every agent shares the ID of the most played arm in their active set during phase $j$, which is what we refer to as the ``best'' performing arm. Similarly, the ``worst'' performing arm in the active set during a phase refers to the non-sticky arm that was played the least number of times in that phase. The intuition behind sharing the most played arm during a phase is that for large time horizons, UCB chooses to play the best arm more than any other arm \cite{bubeck2011pure}. Therefore, if the true best arm is present in the active set, it will be recommended at the end of a phase as the algorithm progresses and the phases grow longer. \\

{\bf{Active set update for the next phase:}} We update the active set in a more efficient manner \cite{newton2022asymptotic} compared to the GosInE Algorithm in \cite{aistats2020gossip}. Specifically, GosInE uses the update $S_{j+1}^{(i)} = \widehat{\mathcal{S}}^{(i)} \cup \{U_{j}^{(i)}\} \cup \{\mathcal{O}_{j}^{(i)}\}$, where $U_{j}^{(i)}$ is the most played non-sticky arm in phase $j$, i.e., $U_{j}^{(i)} = \arg \max_{k \in S_{j}^{(i)} \backslash \widehat{\mathcal{S}}^{(i)}} T_{k}^{(i)}(A_j) - T_{k}^{(i)}(A_{j-1})$. In contrast, we use the update $S_{j+1}^{(i)} = \widehat{\mathcal{S}}^{(i)} \cup \{\widehat{\mathcal{O}}_{j}^{(i)}\} \cup \{\mathcal{O}_{j}^{(i)}\}$, where $\widehat{\mathcal{O}}_{j}^{(i)} = \arg \max_{k \in S_{j}^{(i)}} T_{k}^{(i)}(A_j) - T_{k}^{(i)}(A_{j-1})$ is the most played among \textit{all} arms (not just the non-sticky arms). As observed in \cite{newton2022asymptotic}, the latter update ensures that once the best arm spreads, the active set becomes $\widehat{\mathcal{S}}^{(i)} \cup \{k^{*}\}$ thereafter, where $k^*$ is the true best arm. This reduces the cardinality of the active set by up to two arms if ${k^{*}} \in \widehat{\mathcal{S}}^{(i)}$, subsequently reducing the regret in the single bandit case. As our analysis will show, a similar regret reduction is possible for our setting.

\begin{algorithm}[t]
   \caption{(at agent $i$)}
   \label{alg:gosine}
\begin{algorithmic}
   \STATE {\bfseries Input:} UCB Parameter $\alpha > 0$, phase parameter $\beta > 1$, sticky set $\widehat{S}^{(i)}$ with $|\widehat{S}^{(i)}| = S \leq K-2$
   \STATE Initialize $A_{j} = \lceil j^{\beta} \rceil$, $j \leftarrow 1$, $S_{1}^{(i)} \leftarrow \widehat{S}^{(i)}$
   \FOR{$t \in \mathbb{N}$} 
   \STATE Play $I_{t}^{(i)} = \arg \max_{k \in S_{j}^{(i)}} \widehat{\mu}_{k}^{(i)}(t-1) + \sqrt{\frac{\alpha \log T}{T_{k}^{(i)}(t-1)}}$ 
   \IF{$t==A_{j}$}  
   \STATE $\mathcal{O}_{j}^{(i)} \leftarrow$ GetRec$(i,j)$ \hfill(Algorithm \ref{alg:armrec})
   \STATE $\widehat{\mathcal{O}}_{j}^{(i)} \leftarrow \arg \max_{k \in S_{j}^{(i)}} T_{k}^{(i)}(A_j) - T_{k}^{(i)}(A_{j-1})$ 
   \STATE $S_{j+1}^{(i)} \leftarrow \widehat{\mathcal{S}}^{(i)} \cup \{\widehat{\mathcal{O}}_{j}^{(i)}\} \cup \{\mathcal{O}_{j}^{(i)}\}$ 
   \STATE $j \leftarrow j+1$
   \ENDIF
   \ENDFOR
\end{algorithmic}
\end{algorithm}

\begin{algorithm}[h]
   \caption{Arm Recommendation}
   \label{alg:armrec}
\begin{algorithmic}
   \STATE {\bfseries Input:} Agent $i \in [N]$, phase $j \in \mathbb{N}$
   \FUNCTION {GetRec$(i,j)$} 
   \STATE $n \sim G(i,.)$ (sample a neighbor)
   \STATE {\bf{return}} $\widehat{\mathcal{O}}_{j}^{(n)}$ (most played arm by agent $n$ in phase $j$)
   \ENDFUNCTION
\end{algorithmic}
\end{algorithm}

\subsection{Assumption on the Sticky Set}
\label{subsec:stickysetsize}
It is possible that during the initialization of Algorithm \ref{alg:gosine}, none of the agents learning a particular bandit have the best arm $k^{*}(m)$ in their sticky sets, i.e., $k^{*}(m) \notin \widehat{S}^{(i)}$ for all $i \in \mathcal{I}_m$ for some $m \in [M]$.
This will result in all agents learning that bandit to incur linear regret.
In order to avoid this situation, we will follow \cite{aistats2020gossip, vial2021robust,vial2022robust, newton2022asymptotic, sankararaman2019social} and make a mild assumption:
\begin{assume}
\label{assume:stickyset}
    For all $m \in [M]$, there exists an agent $i_{m}^{*} \in \mathcal{I}_{m}$ such that $k^{*}(m) \in \widehat{S}^{(i_{m}^{*})}$.
\end{assume}

{\bf{Remarks:}}

1. In fact, if $S=\left\lceil \frac{MK}{c_1 N} \log \frac{M}{\gamma} \right\rceil$ for some $\gamma \in (0, 1)$ and we construct $\widehat{S}^{(i)}$ for each $i$ by sampling $S$ arms independently and uniformly at random from $K$ arms, then Assumption \ref{assume:stickyset} holds with probability at least $1-\gamma$. We prove this claim as Proposition \ref{prop:bestarmstickyset} in Appendix \ref{appdx:stickysetsize}.
This choice of $S$, scaling as $\frac{MK}{N} = \frac{K}{\left(\frac{N}{M}\right)}$, implies that for every bandit $m \in [M]$, the $K$ arms are equally distributed across the set of $N_m = \Theta(\frac{N}{M})$ agents $\mathcal{I}_{m}$ with high probability, ensuring that every arm remains active for some agent learning that bandit.\\


2. We can alternatively define the set of arms to be those present among their sticky sets of agents to avoid Assumption \ref{assume:stickyset} altogether.
In such a scenario, agents learning a particular bandit will learn and spread the best arm among the arms in their sticky sets. 

\subsection{Regret Guarantee}
Theorem \ref{thm:gosinereg} characterizes the performance of Algorithm \ref{alg:gosine}.

\begin{theorem}
\label{thm:gosinereg}
Consider a system of $N \geq 2$ agents connected by a complete graph (for each $i \in [N]$, $G(i,n) = (N-1)^{-1} \forall n \neq i$) and learning one of the $M \geq 2$ bandits with $K \geq 2$ arms, satisfying Assumption \ref{assume:stickyset}. Let the UCB parameter $\alpha > 10$ and the phase parameter $\beta > 2$. Then, the regret incurred by an agent $i \in \mathcal{I}_{m}$ running Algorithm \ref{alg:gosine} for each $m \in [M]$ after $T$ time steps is bounded by:
\begin{equation*}
    \mathbb{E}[R_{T}^{(i)}] \leq  \lceil (\tau^{*})^{\beta} \rceil + (K + g)\frac{\pi^2}{3} + g_{\mathrm{spr}} + \sum_{k \in \{\widehat{\mathcal{S}}^{(i)} \cup \mathcal{B}^{(-m)}\} \backslash \{k^{*}(m)\}}  \frac{4\alpha}{\Delta_{m, k}} \log T,  
\end{equation*}
where $\tau^{*} = 2\max\{2, \max_{m \in [M]}\tau_{m}^{*}\}$, $\tau_{m}^{*} = \inf \left\{j \in \mathbb{N}: \frac{A_{j} - A_{j-1}}{S+2} \geq 1 + \frac{4\alpha}{\Delta_{m}^2}\log A_{j}\right\}$, 

$g = \frac{N (2^{\beta}+1) 2^{\beta(\frac{\alpha}{2}-3)} (S+1)}{\frac{\alpha}{2}-3} {K \choose 2}$, $g_{\mathrm{spr}}$ scales as $O\left(M^{\beta+1} \left(\left(\log \frac{N}{M}\right)^{2} \log \left(\log \frac{N}{M}\right)\right)^{\beta}\right)$ and $O(.)$ only hides the absolute constants.
\end{theorem}


{\bf Remarks:} 

{\bf 1. Scaling of $\tau^*$} - Proposition \ref{prop:taubound} in Appendix \ref{pf:gosinereg} implies that $\tau^* = O\left(\frac{S}{\Delta^{2}}\right)^{\frac{1}{\beta-2}}$, where $\Delta = \min_{m \in [M]} \Delta_{m}$. This is expected, because it takes longer for the best arm to be identified and spread for bandits with the smaller gaps.\\

{\bf 2. Regret scaling} - Essentially, Theorem \ref{thm:gosinereg} says that the regret of any agent $i \in \mathcal{I}_{m}$ for all $m \in [M]$ scales as $O\left(\frac{S+M}{\Delta_m}\log T\right)$ for $T$ large.\\

{\bf 3. Regret guarantee for arbitrary values of arm means} - The result in Theorem \ref{thm:gosinereg} can be easily extended for arm means not restricted to $[0,1]$, as assumed above.
The only modification that occurs is the following: the terms $K+g$, $g_{\mathrm{spr}}$ and $\lceil (\tau^{*})^{\beta} \rceil$ will be multiplied by $\widetilde{\Delta}_{m} := \max_{k \in [K] \backslash {k^{*}(m)}} \Delta_{m, k}$. 
It is noteworthy that the modification only affects the second-order term in regret.\\

{\bf 4. Benefit of collaboration} - We have the following corollary when the $K$ arms are equally distributed across all the agents learning the same bandit, i.e., when $S = \Theta\left(\frac{MK}{N}\right)$.

\begin{corollary}
\label{cor:gosinereg}
    When $S = \Theta\left(\frac{MK}{N}\right)$, the regret incurred by an agent $i \in \mathcal{I}_{m}$ for each $m \in [M]$ after $T$ time steps scales as $O\left(\frac{1}{\Delta_m}\left(\frac{MK}{N}+M\right) \log T\right)$.
\end{corollary}
Corollary \ref{cor:gosinereg} implies that when $S = \Theta\left(\frac{MK}{N}\right)$ and $M << \min \{K, N\}$, agents experience regret reduction compared to the case of no collaborations.\\

{\bf 5. Group Regret} - Corollary \ref{cor:gosinegpreg} quantifies the performance of Algorithm \ref{alg:gosine} in terms of the group regret.

\begin{corollary}
\label{cor:gosinegpreg}
    For all bandits $m \in [M]$, when $\{\widehat{S}^{(i)}\}_{i \in \mathcal{I}_{m}}$ is a partition of the set of $K$ arms, i.e., $\widehat{S}^{(i_1)} \cap \widehat{S}^{(i_2)} = \phi$ for $i_1 \neq i_2 \in \mathcal{I}_{m}$ and $\cup_{i \in \mathcal{I}_{m}} \widehat{S}^{(i)} = [K]$, the group regret $\mathrm{Reg}(T)$ of the system playing Algorithm \ref{alg:gosine} satisfies
    \begin{multline*}
     \mathbb{E}[\mathrm{Reg}(T)] \leq \sum_{m \in [M]} \sum_{k \in [K] \backslash \{k^{*}(m)\}} \frac{4\alpha}{\Delta_{m, k}} \log T + c_2 \frac{N}{M} \sum_{m \in [M]} \sum_{k \in \mathcal{B}^{(-m)}} \frac{4\alpha}{\Delta_{m, k}} \log T \\
     + N\lceil (\tau^{*})^{\beta} \rceil + N(K + g)\frac{\pi^2}{3} + Ng_{\mathrm{spr}}. 
    \end{multline*}
\end{corollary}
Essentially, Corollary \ref{cor:gosinegpreg} implies expected group regret scales as $O\left(\sum_{m \in [M]}\frac{K+N}{\Delta_m}\log T\right)$ for large $T$.

\subsection{Proof Sketch (Theorem \ref{thm:gosinereg})}
\label{pfsk:gosine}
The proof of Theorem \ref{thm:gosinereg} is detailed in Appendix \ref{pf:gosinereg} and we highlight the key ideas here.
Akin to the scenario of learning the single bandit in \cite{aistats2020gossip}, we first show the existence of a random phase $\tau$, after which all the agents starting from phase $\tau$ contain the best arm of the bandit they are learning in their respective active sets and recommend it during information pulls (Claim \ref{prop:freeze}).
This allows us to decompose the expected cumulative regret incurred by an agent into two parts: the regret up to phase $\tau$ and the regret after phase $\tau$. 
The regret after phase $\tau$ is the regret incurred by playing the UCB algorithm on the sticky set and the $M$ best arms, because agents recommend their respective best arms during information pulls post phase $\tau$.\\


The technical challenge lies in bounding the expected cumulative regret up to phase $\tau$.
In particular, we prove a generalization of the result in \cite{aistats2020gossip} that irrespective of the bandit an agent is learning, the probability of an agent not recommending their best arm and thus dropping it from their active set at the end of a phase is small and decreases as the phases progress, such that it doesn't happen infinitely often (Lemma \ref{lem:errorprob}).
This is a consequence of agents playing UCB Algorithm on their active sets during a phase and the fact that UCB chooses to play the best arm more often than any other arm for large time horizons \cite{bubeck2011pure}.
This implies the existence of a random phase (denoted by $\tau_{\mathrm{stab}}$), after which agents aware of their best arm (i.e., agents with the relevant best arm in their active set) will always recommend it moving forward.\\

Our analysis differs significantly from the analysis in \cite{aistats2020gossip} post phase $\tau_{\mathrm{stab}}$, where we characterize the time $\tau_{\mathrm{spr}, m}$ required for the best arm of the $m^{\mathrm{th}}$ bandit to spread via recommendations to the agents $\mathcal{I}_m$ learning that bandit. By definition of $\tau_{\mathrm{stab}}$, we know that if an agent $i_1 \in \mathcal{I}_{m}$ contacts another agent $i_2 \in \mathcal{I}_{m}$ who is aware of the true best arm $k^{*}(m)$ after phase $\tau_{\mathrm{stab}}$, then $i_1$ will become informed of the true best arm as well. This arm spreading process therefore resembles a rumor process, where one agent ($i_m^*$ in Assumption \ref{assume:stickyset}) initially knows the rumor (i.e., the best arm), and any agent who contacts someone aware of the rumor becomes informed itself.\\

However, our rumor process is extremely complicated compared to the process for the single bandit case in \cite{aistats2020gossip}.
This is because we have $M$ intertwined rumor spreading processes evolving simultaneously: every agent can interact with agents learning a different bandit, and post phase $\tau_{\mathrm{stab}}$, agents recommend whichever of the $M$ rumors (i.e., best arms) is relevant to them.
Hence, none of these $M$ rumor spreading processes are standard rumor spreading processes (unlike \cite{aistats2020gossip}), so analyzing them directly is infeasible.\\

To tackle this issue, we disentangle the $M$ intertwined processes via a coupling argument. In particular, we define $M$ independent \textit{noisy} rumor processes and show via coupling that the spreading times of these processes upper bound the time of the actual arm spreading processes. The $m^{\mathrm{th}}$ of the noisy rumor processes, denoted by $\{\bar{\mathcal{R}}_{m, j}\}_{j=0}^{\infty}$, unfolds on the subgraph of agents $\mathcal{I}_m$ learning the $m^{\mathrm{th}}$ bandit and tracks the agents $\bar{\mathcal{R}}_{m, j}$ aware of the $m^{\mathrm{th}}$ rumor at phase $j$.
Initially, only $i_{m}^{*}$ is aware, i.e., $\bar{\mathcal{R}}_{m, 0} = \{i_{m}^{*}\}$ by Assumption \ref{assume:stickyset}.
In each phase $j \in \mathbb{N}$, each agent $i \in \mathcal{I}_{m}$ contacts another agent $ag \in \mathcal{I}_{m}$ uniformly at random.
If $ag \in \bar{\mathcal{R}}_{m, j-1}$ ($ag$ is aware of the rumor), then $i \in \bar{\mathcal{R}}_{m, j}$ ($i$ becomes aware as well) subject to Bernoulli$(\eta)$ noise, where $\eta = \frac{N_m-1}{N-1}$. 
Therefore, $i$ becomes aware with probability $|\bar{\mathcal{R}}_{m, j-1} \cap \mathcal{I}_{m}|\eta/N_m \leq |\bar{\mathcal{R}}_{m, j-1} \cap \mathcal{I}_{m}|/(N-1)$.
Observe that the right hand side of the inequality is equal the probability with which agent $i$ becomes aware of the best arm in the real process (since in the real process, $i$ contacts $ag \in [N] \setminus \{i\}$ uniformly at random), which allows us to upper bound the spreading time via coupling as discussed above. Thereafter, we further couple the noisy processes to a noiseless one as in \cite{vial2022robust}, then use an existing bound for the noiseless setting \cite{lattanzi}.

\section{\emph{Partially Context Aware} Algorithm}
\label{sec:gosinesideinfo}
From Corollary \ref{cor:gosinereg}, it can be noticed that Algorithm \ref{alg:gosine} incurs additional regret scaling of $O\left(\frac{M}{\Delta_m} \log T\right)$ after $T$ time steps.
This is because agents playing Algorithm \ref{alg:gosine} have the best arm in their active sets and recommend it during information pulls after a random phase $\tau$.
Hence, in the absence of knowledge about which other agents are learning the same bandit, any agent playing Algorithm \ref{alg:gosine} is unable to determine this information with certainty, due to the random nature of $\tau$.
One can think of ways in which agents can stop communicating with the agents not learning the same bandit, for example, the blocking approaches considered in \cite{vial2021robust, vial2022robust}.
These works considered agents learning a single bandit collaboratively in the presence of adversarial agents.
However, such blocking strategies incur worse regret: under the assumptions of Corollary \ref{cor:gosinereg}, both \cite{vial2021robust, vial2022robust} incur regret of $O\left(\frac{1}{\Delta_m}\left(\frac{MK}{N} + N\left(1-\frac{1}{M}\right)\right) \log T\right)$ for large $T$.
This regret scaling is problematic when $\frac{K}{N} = \Theta(1)$, as there is no benefit of collaboration: the regret scales as $O\left(\frac{K}{\Delta_m} \log T\right)$, which is akin to learning a single agent regret. \\

Given that agents are collaborative in our setting in the sense that they divide the exploration of sub-optimal arms in addition to identifying their best arm, Algorithm \ref{alg:gosine} has a structure to it: post phase $\tau$, there are only $M$ rumors (corresponding to the $M$ best arms) flowing through the network.
If each agent is aware of $r-1$ other agents learning the same bandit, they can distribute the exploration of the received arm recommendations as follows: after receiving the arm recommendations, each group of $r$ agents learning the same bandit can select the $M$ most recent unique arm recommendations from all the recommendations they have received so far and divide them (almost) equally among themselves. 
The intuition is that post phase $\tau$, given that there are only $M$ rumors flowing through the network, we know from the coupon collector problem that after a (finite) random number of phases post phase $\tau$, the $M$ most recent unique arm recommendations will be the $M$ best arms, and it will stay that way from then on.
This reduces the additional regret due to arm recommendations by a factor of $r$.\\

We use the intuition in the previous paragraph and propose Algorithm \ref{alg:gosinesideinfo}, which builds upon Algorithm \ref{alg:gosine} and uses the following extra input: for each $m \in [M]$ and $i \in \mathcal{I}_{m}$, let $f(i)$ satisfy: (i) $f(i) \subset \mathcal{I}_{m}$, (ii) $|f(i)| = r-1$, and (iii) if $n \in f(i)$, then $i \in f(n)$.
Thus, $f(i)$ consists of $r-1$ other agents learning the same bandit who are known to agent $i$.\\

Algorithm \ref{alg:gosinesideinfo} divides the $M$ most recent unique arm recommendations among $r$ agents in $\{i\} \cup f(i)$ using the subroutine DivideRec, described in Algorithm \ref{alg:uniquerecdist}.
DivideRec can be best understood for the case when $\frac{M}{r}$ is a (positive) integer.
For example, if $M=6$ and $r=3$, each agent in the set $\{i\} \cup f(i)$ will get $2$ arm IDs from $\mathrm{sortrec}(i,f(i),.)$.
Suppose that $i$ is the second smallest element in the sorted version of $\{i\} \cup f(i)$ ($pos(i)=2$), it will get the third and the fourth entries in $\mathrm{sortrec}(i,f(i),.)$.\\

It is important to note that the array $\mathrm{sortrec}(i,f(i),.)$ is identical for all $ag \in \{i\} \cup f(i)$.
This observation is crucial in dividing the $M$ most recent unique recommendations among the $r$ agents in $\{i\} \cup f(i)$, without violating the constraint on the number of communication bits per agent.\\
\\
\\


\begin{algorithm}[t]
   \caption{(at agent $i$)}
   \label{alg:gosinesideinfo}
\begin{algorithmic}
   \STATE {\bfseries Input:} UCB Parameter $\alpha > 0$, phase parameter $\beta > 1$, sticky set $\widehat{S}^{(i)}$ with $|\widehat{S}^{(i)}| = S \leq K-2-\lceil\frac{M}{r}\rceil$, the set $f(i)$ of agents known to be learning the same bandit
   \STATE Initialize $A_{j} = \lceil j^{\beta} \rceil$, $j \leftarrow 1$, $S_{1}^{(i)} \leftarrow \widehat{S}^{(i)}$, $\mathrm{rec}(ag) = \{\}$ for all $ag \in i \cup f(i)$.
   \FOR{$t \in \mathbb{N}$} 
   \STATE Play $I_{t}^{(i)} = \arg \max_{k \in S_{j}^{(i)}} \widehat{\mu}_{k}^{(i)}(t-1) + \sqrt{\frac{\alpha \log T}{T_{k}^{(i)}(t-1)}}$ 
   \IF{$t==A_{j}$}  
   \STATE $\mathcal{O}_{j}^{(i)} \leftarrow$ GetRec$(i,j)$ \hfill(Algorithm \ref{alg:armrec})
    \STATE $\widehat{\mathcal{O}}_{j}^{(i)} \leftarrow \arg \max_{k \in S_{j}^{(i)}} T_{k}^{(i)}(A_j) - T_{k}^{(i)}(A_{j-1})$
   \STATE $\mathrm{rec}(i) \leftarrow \mathrm{rec}(i) \cup \{(j, \mathcal{O}_{j}^{(i)})\}$ 
   \STATE Obtain $\mathcal{O}_{j}^{(ag)}$ from all $ag \in f(i)$ to maintain the set of arm recommendations $\mathrm{rec}(ag) \leftarrow \mathrm{rec}(ag) \cup \{(j, \mathcal{O}_{j}^{(ag)})\}$ for all $ag \in f(i)$
   \STATE $\mathrm{uniquerec}(i,f(i),j) \leftarrow$ $M$ most recent unique arm recommendations in $\bigcup_{ag \in \{\{i\} \cup f(i)\}} \mathrm{rec}(ag)$
   \IF{$\mathrm{uniquerec}(i,f(i),j) \neq \mathrm{uniquerec}(i,f(i),j-1)$}
   \STATE $\mathrm{sortrec}(i,f(i),j) \leftarrow$ elements of $\mathrm{uniquerec}(i,f(i),j)$ in the ascending order
   \STATE $\widetilde{S}_{j}^{(i)} \leftarrow$ DivideRec$(i,j,f(i),\mathrm{sortrec}(i,f(i),j)))$\hfill (Algorithm \ref{alg:uniquerecdist})
  \STATE $S_{j+1}^{(i)} \leftarrow \widehat{\mathcal{S}}^{(i)} \cup \{\widehat{\mathcal{O}}_{j}^{(i)}\} \cup \{\mathcal{O}_{j}^{(i)}\} \cup \{\widetilde{S}_{j}^{(i)}\}$
  \ELSE
  \STATE $S_{j+1}^{(i)} \leftarrow S_{j}^{(i)}$
  \ENDIF
   \STATE $j \leftarrow j+1$
   \ENDIF
   \ENDFOR
\end{algorithmic}
\end{algorithm}

\begin{algorithm}[h]
   \caption{Dividing $M$ most recent unique arm recommendations}
   \label{alg:uniquerecdist}
\begin{algorithmic}
   \STATE {\bfseries Input:} Agent $i \in [N]$, phase $j \in \mathbb{N}$, sets $f(i)$ and $\mathrm{sortrec}(i,f(i),j))$
   \FUNCTION {DivideRec$(i,j,f(i),\mathrm{sortrec}(i,f(i),j)))$} 
   \STATE $ags \leftarrow$ elements of $\{i\} \cup f(i)$ in the ascending order
   \STATE $pos(i) \leftarrow$ position of $i$ in the array $ags$
   \STATE $\widetilde{S}_{j}^{(i)} \leftarrow$ elements of $\mathrm{sortrec}(i,f(i),j))$ in the positions $\Big{\{}\left((pos(i)-1)\left\lceil\frac{M}{r}\right\rceil\; \mathrm{mod}\; M\right) + 1, \cdots, $ 
   \STATE $\left(pos(i)\left\lceil\frac{M}{r}\right\rceil - 1 \;\mathrm{mod}\; M\right) + 1 \Big{\}}$
   \STATE {\bf{return}} $\widetilde{S}_{j}^{(i)}$
   \ENDFUNCTION
\end{algorithmic}
\end{algorithm}


{\bf Remarks:}

{\bf 1. Constructing $\widetilde{S}_{j}^{(.)}$ when $|\mathrm{uniquerec}(.,f(.),j)|<M$ -} when enough phases haven't elapsed such that there are less than $M$ most recent unique arm recommendations, we can construct $\widetilde{S}_{j}^{(.)}$ with  $\lceil\frac{|\mathrm{uniquerec}(.,f(.),j)|}{r}\rceil$ elements.\\

{\bf 2. Satisfying the communication bit constraint of $r\lceil\log_{2}K\rceil$ bits -} Each agent $i \in \mathcal{I}_{m}$ uses $\lceil\log_{2}K\rceil$ bits to receive an arm recommendation via an information pull, and uses $\lceil\log_{2}K\rceil$ bits per agent to obtain the arm recommendations received by $r-1$ agents in $f(i)$.

\subsection{Regret Guarantee}
Theorem \ref{thm:gosinesideinforeg} characterizes the performance of Algorithm \ref{alg:gosinesideinfo}. Here, for any bandit $m \in [M]$, $k_{m,1}, k_{m,2}, \cdots, k_{m,K}$ denotes the order statistics of the arm means, i.e., $k_{m,1} = k^{*}(m)$ and $\mu_{m,k_{m,1}} > \mu_{m,k_{m,2}} \geq \cdots \geq \mu_{m, k_{m,K}}$.
\begin{theorem}
\label{thm:gosinesideinforeg}
Under the assumptions of Theorem \ref{thm:gosinereg}, the regret incurred by an agent $i \in \mathcal{I}_{m}$ running Algorithm \ref{alg:gosinesideinfo} for each $m \in [M]$ after $T$ time steps is bounded by:
\begin{multline*}
    \mathbb{E}[R_{T}^{(i)}] \leq  \lceil (j^{*})^{\beta} \rceil + (K + \widehat{g})\frac{\pi^2}{3} + \widehat{g}_{\mathrm{spr}} + \widehat{g}_{\mathrm{rec}} + \sum_{k \in \widehat{S}^{(i)} \backslash \{k^{*}(m)\}} \frac{4\alpha}{\Delta_{m, k}} \log T + \sum_{k \in \{k_{m,l}\}_{l=2}^{\lceil\frac{M}{r}\rceil+2}} \frac{4\alpha}{\Delta_{m, k}} \log T  
\end{multline*}
where  $j^{*} = 3\max\{2, \max_{m \in [M]}j_{m}^{*}\}$, 
$j_{m}^{*} = \inf \left\{j \in \mathbb{N}: \frac{A_{j} - A_{j-1}}{S+2+\lceil \frac{M}{r} \rceil} \geq 1 + \frac{4\alpha}{\Delta_{m}^2}\log A_{j}\right\}$, 

$\widehat{g}= \frac{N (3^{\beta}+1) 3^{\beta(\frac{\alpha}{2}-3)} (S+1+\lceil \frac{M}{r} \rceil)}{\frac{\alpha}{2}-3} {K \choose 2+\lceil \frac{M}{r} \rceil}$,  
$\widehat{g}_{\mathrm{rec}} = \frac{N}{r}\left(\lceil (3M)^{\beta} \rceil + 2 \left(\frac{3}{\left(\frac{c_1}{M}-\frac{1}{N}\right)r}\right)^{\beta} \frac{M}{\left(1-\frac{c_1}{M}\right)^{r}} \Gamma(\beta+1)\right)$, $\Gamma(z) = \int_{t=0}^{\infty} t^{z-1} e^{-t}\; \mathrm{d}t$ for $z>0$, 
$\widehat{g}_{\mathrm{spr}}$ scales as $O\left(M^{\beta+1} \left(\left(\log \frac{N}{M}\right)^{2} \log \left(\log \frac{N}{M}\right)\right)^{\beta}\right)$ and $O(.)$ only hides the absolute constants.
\end{theorem}

{\bf Remarks:} 

{\bf 1. Scaling of $j^*$} - $j^{*}$ scales just like $\tau^{*}$ in Theorem \ref{thm:gosinereg}, except with $S$ replaced by $S+\frac{M}{r}$. Proposition \ref{prop:tausideinfobound} in Appendix \ref{pf:gosinesideinforeg} formalizes this scaling.\\

{\bf 2. Regret Scaling} - Essentially, Theorem \ref{thm:gosinesideinforeg} says that the regret of any agent $i \in \mathcal{I}_{m}$ for all $m \in [M]$ scales as $O\left(\frac{S+\frac{M}{r}}{\Delta_m}\log T\right)$ for large $T$.\\

{\bf 3. Benefit of collaboration} - We have the following corollary when the $K$ arms are equally distributed across all the agents learning the same bandit, i.e., when $S = \Theta\left(\frac{MK}{N}\right)$.
\begin{corollary}
\label{cor:gosinesideinforeg}
    When $S = \Theta\left(\frac{MK}{N}\right)$, the regret incurred by an agent $i \in \mathcal{I}_{m}$ for each $m \in [M]$ after $T$ time steps scales as $O\left(\frac{1}{\Delta_m}\left(\frac{MK}{N}+\frac{M}{r}\right) \log T\right)$.
\end{corollary}
It is clear from Theorem \ref{thm:gosinesideinforeg} and Corollary \ref{cor:gosinesideinforeg} that agents having knowledge of $r-1$ other agents learning the same bandits results in lesser regret, compared to the \emph{context unaware} scenario where agents aren't aware of the other agents are learning the same bandit.\\

{\bf 4. Group Regret} - Corollary \ref{cor:gosinesideinfogpreg} quantifies the performance of Algorithm \ref{alg:gosine} in terms of the group regret.

\begin{corollary}
\label{cor:gosinesideinfogpreg}
    For all bandits $m \in [M]$, when $\{\widehat{S}^{(i)}\}_{i \in \mathcal{I}_{m}}$ is a partition of the set of $K$ arms, i.e., $\widehat{S}^{(i_1)} \cap \widehat{S}^{(i_2)} = \phi$ for $i_1 \neq i_2 \in \mathcal{I}_{m}$ and $\cup_{i \in \mathcal{I}_{m}} \widehat{S}^{(i)} = [K]$, the group regret $\mathrm{Reg}(T)$ of the system playing Algorithm \ref{alg:gosinesideinfo} satisfies
    \begin{multline*}
     \mathbb{E}[\mathrm{Reg}(T)] \leq \sum_{m \in [M]} \sum_{k \in [K] \backslash \{k^{*}(m)\}} \frac{4\alpha}{\Delta_{m, k}} \log T + c_2 \frac{N}{M} \sum_{m \in [M]} \sum_{k \in \{k_{m,l}\}_{l=2}^{\lceil\frac{M}{r}\rceil+2}} \frac{4\alpha}{\Delta_{m, k}} \log T \\
     + N\lceil (j^{*})^{\beta} \rceil + N(K + \widehat{g})\frac{\pi^2}{3} + N\widehat{g}_{\mathrm{spr}} + N\widehat{g}_{\mathrm{rec}}. 
    \end{multline*}
\end{corollary}
Essentially, Corollary \ref{cor:gosinesideinfogpreg} implies that agents running Algorithm \ref{alg:gosinesideinfo} incur expected group regret scaling as $O\left(\sum_{m \in [M]}\frac{K+\frac{N}{r}}{\Delta_m}\log T\right)$ for large $T$.

\subsection{Proof Sketch (Theorem \ref{thm:gosinesideinforeg})}
Similar to the regret analysis of Algorithm \ref{alg:gosine}, we first show the existence of (finite) random phases: (i) $\tau_{\mathrm{stab}}$, such that if agents have the best arm in their active sets, that will be their most played arm and recommended henceforth during information pulls, and (ii) additional $\tau_{\mathrm{spr}}$ phases post the phase $\tau_{\mathrm{stab}}$, after which all the agents have their best arms in their active sets.
The proof for showing that the random phases $\tau_{\mathrm{stab}}$ and $\tau_{\mathrm{spr}}$ are finite proceeds identically to the proof for Theorem \ref{thm:gosinereg}. \\

Moreover, for Algorithm \ref{alg:gosinesideinfo}, we also show the existence of additional (finite) random phases post the phase $\tau_{\mathrm{stab}} + \tau_{\mathrm{spr}}$, denoted by $\tau_{\mathrm{rec}}$, after which the $M$ most recent unique arm recommendations is equal to the set of $M$ best arms from then onwards. 
As a consequence of the active set updates in Algorithm \ref{alg:gosinesideinfo}, the active sets of agents remain unchanged in all the subsequent phases and freeze like the GosInE Algorithm in \cite{aistats2020gossip}, which helps improve the regret by distributing the exploration of $M$ best arms across $r-1$ other agents learning the same bandit.\\

{\bf Remark:} This freezing does not happen in Algorithm \ref{alg:gosine}, where the active sets are still time-varying post phase $\tau$ and the randomness of $\tau$ prevents agents from gathering information about others learning the same bandit with certainty.


\section{Lower Bounds}
\label{sec:lowerbound}
We state lower bounds for Gaussian noise with unit variance. As is standard, we restrict to \emph{uniformly efficient} policies, i.e., those that ensure small regret on any problem instance. In our case, instances are defined by the arm means $\mu = \{ \mu_{m,k} \}_{m \in [M] , k \in [K] }$ of the $M$ bandits and the partition $\mathcal{I} = \{ \mathcal{I}_m \}_{m \in [M]}$ of the $N$ agents to the $M$ bandits. Policies are called uniformly efficient if $\mathbb{E}[\mathrm{Reg}_{\mu,\mathcal{I}}(T)] = o ( T^\gamma )$ for any $\gamma \in (0,1)$ and any instance $(\mu,\mathcal{I})$.

\begin{theorem}
 \label{thm:lowerbound}
 For any uniformly efficient policy,
 \begin{align*}
     \liminf_{T \rightarrow \infty} \frac{\mathbb{E}[\mathrm{Reg}_{\mu,\mathcal{I}}(T)]}{ \log T } \geq \sum_{m \in [M]} \sum_{k \in [K] \backslash \{k^{*}(m)\}} \frac{2}{\Delta_{m,k}}.
 \end{align*}
\end{theorem}

Theorem \ref{thm:lowerbound} is proved in Appendix \ref{pf:lowerbound}, by reducing our model to the setting of \cite{reda2022near}. The result shows that the first terms in Corollaries \ref{cor:gosinegpreg} and \ref{cor:gosinesideinfogpreg} are optimal up to constant factors. Hence, for large $T$, the suboptimality of our upper bounds is due to the second terms, which grow linearly in $N$ and logarithmically in $T$. Under a further assumption on $\mu$, we can show these dependencies are unavoidable. Again, see Appendix \ref{pf:lowerbound} for a proof.

\begin{theorem} \label{thm:lowerbound2}
Let $(\mu,\mathcal{I})$ be any instance satisfying $\mu_{m,k^*(m)} = \mu_{m+1,k^*(m)}$ and $k^*(m) \neq k^*(m+1)$ for each $m \in [M-1]$.\footnote{Such instances $\mu \in [0,1]^{M \times K}$ exist; for example, if $\mu_{m,k} = ( (m-1) + \mathbf{1}(m=k) ) / M$, where $\mathbf{1}$ is the indicator function.\vspace{-2em}} Then for any uniformly efficient policy in the context unaware scenario,
\begin{equation*}
\liminf_{T \rightarrow \infty} \frac{\mathbb{E}[\mathrm{Reg}_{\mu,\mathcal{I}}(T)]}{ \log T } \geq 2 \sum_{m = 1}^{M-1} | \mathcal{I}_m | \Delta_m \geq N \Delta .
\end{equation*}
\end{theorem}



\section{Numerical Results}
\label{sec:simulations}
\begin{figure}[tb]
\centering
\includegraphics[width=\columnwidth]{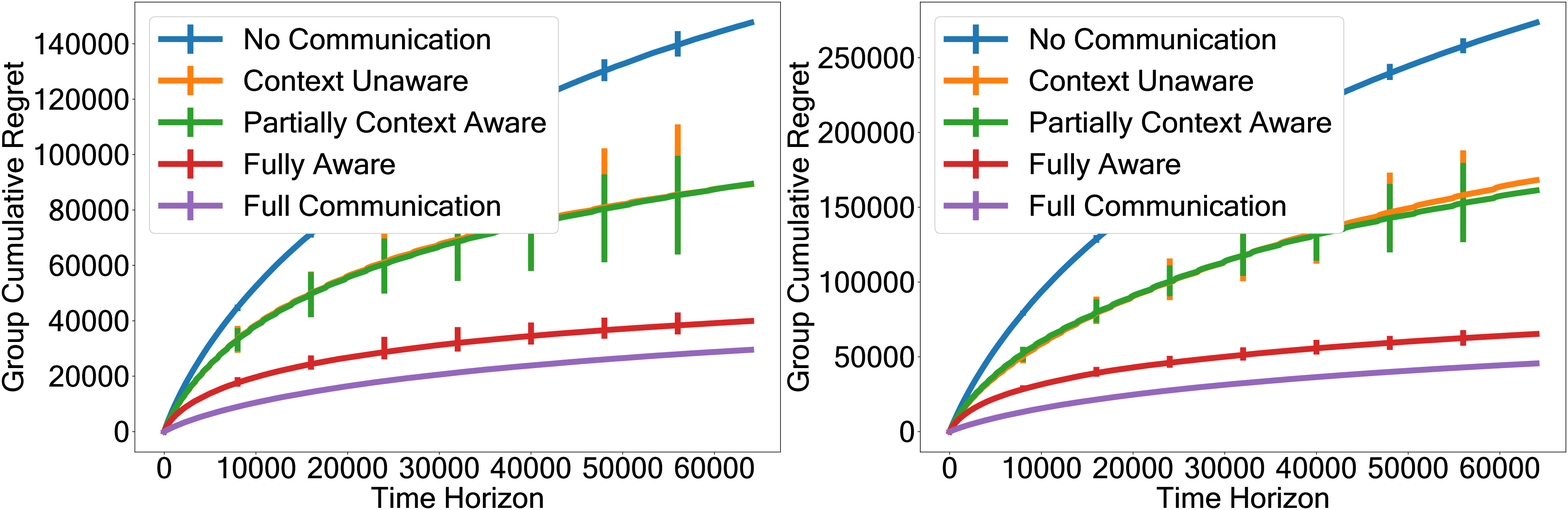}
\caption{$(K, M, N, r)$ are $(20, 5, 25, 5)$ and $(30, 6, 36, 6)$ respectively. Arm means are in $[0, 1)$ and the UCB parameter $\alpha=15$.}
\label{fig1}
\end{figure}

We evaluate Algorithm \ref{alg:gosine} in the \emph{context unaware} setting and Algorithm \ref{alg:gosinesideinfo} in the \emph{partially context aware} setting, and verify their insights through synthetic simulations.
The algorithms are compared with respect to the following benchmarks: (a) no communication, corresponding to all the agents running UCB-$\alpha$ algorithm on $K$-armed MAB from \cite{ACBF02} without any communications, (b) fully aware, corresponding to agents playing GosInE algorithm from \cite{aistats2020gossip}, but agents only communicate with all the other agents learning the same bandit, and (c) full communication, where for each bandit, all agents play the UCB-$\alpha$ algorithm on $K$-armed MAB from \cite{ACBF02}, but with the entire history of all arms pulled and the corresponding rewards obtained by all the agents.

We show the group cumulative regret of Algorithms \ref{alg:gosine} and \ref{alg:gosinesideinfo} over 30 random runs with $95\%$ confidence intervals.
The $K$ arm means for each of the $M$ bandits are generated uniformly at random from $[0, 1)$ in Figure \ref{fig1} and $[2, 4)$ in Figure \ref{fig2}.
We assume an equal number ($\frac{N}{M}$) of agents learning each bandit, set $\beta=3$ and the size of the sticky set $S = \frac{MK}{N}$ in these simulations.
The UCB parameter $\alpha$ is set to $15$ in Figure \ref{fig1} and $30$ in Figure \ref{fig2}. 

From our simulations, it is evident that Algorithms \ref{alg:gosine} and \ref{alg:gosinesideinfo} incur lower regret than the case when agents don't communicate, despite limited communication among the agents and agents interacting with agents learning other bandits.
Furthermore, our simulations also demonstrate that for each bandit, when an agent knows other agents learning the same bandit in the \emph{partially context aware} scenario, it incurs lesser regret compared to the \emph{context unaware} scenario. 

\begin{figure}[tb]
\centering
\includegraphics[width=\columnwidth]{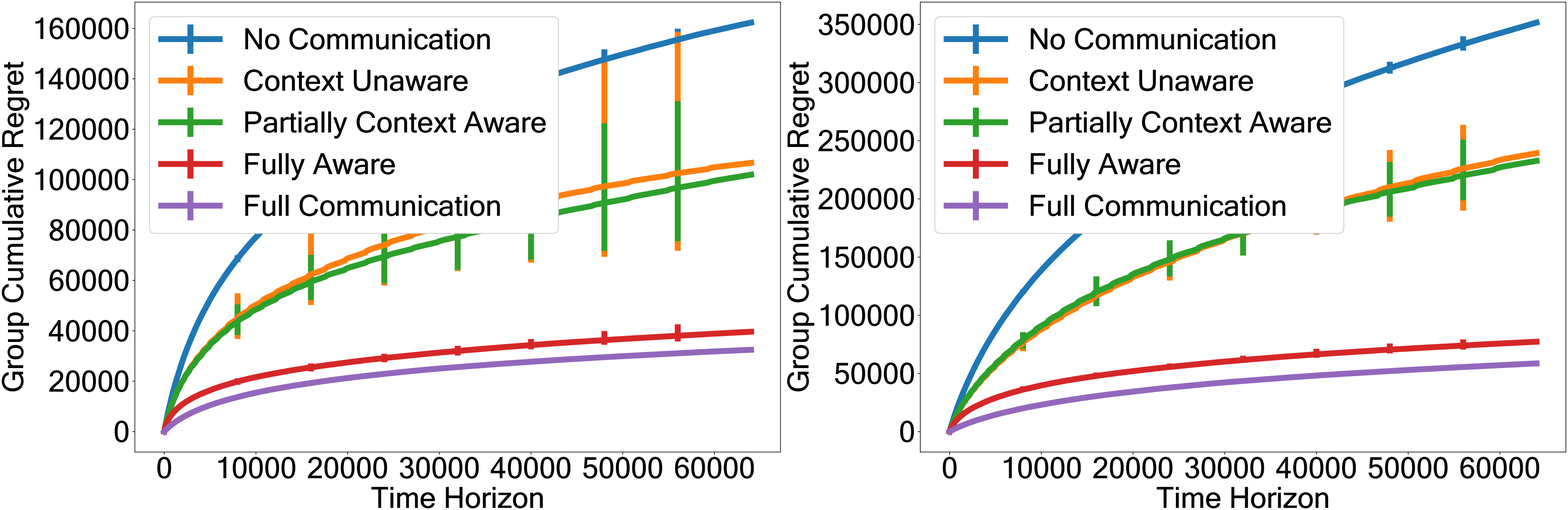}
\caption{$(K, M, N, r)$ are $(20, 5, 25, 5)$ and $(30, 6, 36, 6)$ respectively. Arm means are in $[2, 4)$ and the UCB parameter $\alpha=30$.}
\label{fig2}
\end{figure}

\section*{Acknowledgements}
This work was supported by NSF Grants 2207547, 1934986, 2106801, 2019844, 2112471, and 2107037; ONR Grant N00014-19-1-2566; the Machine Learning Lab (MLL) at UT Austin; and the Wireless Networking and Communications Group (WNCG) Industrial Affiliates Program.



\bibliographystyle{plain}
\bibliography{references}

\clearpage
\begin{appendices}
\section{Related Work Revisited}
\label{appdx:relatedwork}
Apart from the works mentioned in Section \ref{subsec:relatedwork}, there exist several other works in the space of collaborative multi-agent MABs with different models of communication among agents, some of which are shared here.
Agents exchange information in \cite{buccapatnam2015information, chakraborty2017coordinated} via broadcasting instead of pairwise gossip style communications. 
There is more frequent communication between the agents in \cite{kolla2018collaborative, lalitha2021bayesian, rubio}. 
In \cite{madhushani2021call} the number of communications is inversely proportional to the minimum arm gap so could be large. 
Arm mean estimates are exchanged instead of arm indices in \cite{landgren2016distributed}. 
\cite{wang2020optimal} employs a leader-follower mechanism, wherein the leader explores the arms and estimates their mean rewards, while the followers play the arm with the highest estimated arm mean based on the samples collected by the leader.

Collaborative multi-agent bandits have also been studied in the contextual setting \cite{dubey2020kernel,tekin2015distributed} and the linear reward model setting \cite{dubey2020, chawla2022multi,korda2016distributed}, which are significantly different from the setting considered in our work.
Some of the works in this space also focus on minimizing the simple regret \cite{hillel2013distributed,szorenyi2013gossip}, instead of minimizing the cumulative regret.
There also exist works such as \cite{bistriz, shahrampour2017multi, zhuzh, reda2022near} which assume different bandits across agents.
While this list of works is by no means exhaustive, it represents a broad class of settings studied in the collaborative multi-agent bandits thus far.

Another long line of work in multi-agent MABs comprises of collision-based models \cite{anandkumar2011distributed, avner2014concurrent, bistriz, boursier2019sic, dakdouk2021collaborative, kalathil2014decentralized, liu2010distributed, liu2020competing, mansour2017competing, rosenski2016multi}, where if multiple agents play the same arm, they receive a small or no reward.
However, our work assumes that if multiple agents learning the same bandit play the same arm, they receive independent rewards and thus, is different than the collision-based models.

\section{Notations Revisited}
Recall that for every bandit $m \in [M]$ and arm $k \in [K]$, the (unknown) mean rewards are denoted by $\mu_{m, k}$, $k^{*}(m)$ denotes the best arm, $\Delta_{m, k} := \mu_{m, k^{*}(m)} - \mu_{m, k}$ denotes the arm gap, $\Delta_{m} := \min_{k \neq k^{*}(m)} \Delta_{m, k}$ is the minimum arm gap, $\mathcal{I}_{m} \subset [N]$ denotes the set of agents learning the $m^{\mathrm{th}}$ bandit and $c(.):[N] \rightarrow [M]$ is a function mapping the set of agents to the set of bandits, i.e., if $i \in \mathcal{I}_m$, $c(i) = m$. Let $\mathcal{B}$ denote the set of all the $M$ best arms $\{k^{*}(m)\}_{m \in [M]}$ and $\mathcal{B}^{(-m)} = \mathcal{B} \backslash \{k^{*}(m)\}$ denote the set of best arms excluding the best arm of the $m^{\mathrm{th}}$ bandit. For every agent $i \in [N]$ and phase $j \in \mathbb{N}$, $\widehat{\mathcal{O}}_{j}^{(i)} \in S_{j}^{(i)}$ denotes the most played arm by agent $i$ in phase $j$ (where $S_{j}^{(i)}$ is the set of active arms for agent $i$ in phase $j$), and subsequently recommended at the end of the phase if requested by another agent through information pull. Each phase $j$ lasts from $t \in \{1+A_{j-1}, \cdots, A_{j}\}$ and
\begin{equation}
    \label{eq:commepochinv}
    A^{-1}(t) = \sup\{j \in \mathbb{N}: A_{j} \leq t\}.
\end{equation}
For $A_j = \lceil j^\beta \rceil$, $A^{-1}(t) = \lfloor t^{\frac{1}{\beta}} \rfloor$.

\section{Proof of Theorem \ref{thm:gosinereg}}
\label{pf:gosinereg}
We extend the proof ideas from \cite{aistats2020gossip} and \cite{vial2022robust} to prove our results. 
Fix an agent $i \in [N]$ and phase $j \in \mathbb{N}$.
Let $\chi_{j}^{(i)}$ be a boolean random variable associated with the event $\left\{k^{*}(c(i)) \in S_{j}^{(i)}, \widehat{\mathcal{O}}_{j}^{(i)} \neq k^{*}(c(i))\right\}$, i.e.,
\begin{equation}
\label{eq:errorevent}
    \chi_{j}^{(i)} = \mathbf{1}\left(k^{*}(c(i)) \in S_{j}^{(i)}, \widehat{\mathcal{O}}_{j}^{(i)} \neq k^{*}(c(i))\right),
\end{equation}
indicating whether agent $i$ does not recommend the best arm of the bandit it plays at the end of the phase $j$, \emph{if} it is present in its active set $S_{j}^{(i)}$.
We also extend the definitions of random times defined in \cite{aistats2020gossip}, which will capture the key aspects in the system dynamics, as well as highlight the differences from the single bandit case, as follows: 
\begin{align*}
    \tau_{\mathrm{stab}}^{(i)} &= \inf\{ j \geq \tau^{*}: \forall l \geq j, \chi_{l}^{(i)} = 0\}, \\
    \tau_{\mathrm{stab}} &= \max_{i \in [N]} \tau_{\mathrm{stab}}^{(i)}, \\ 
    \tau_{\mathrm{spr}}^{(i)} &= \inf \{j \geq \tau_{\mathrm{stab}}: k^{*}(c(i)) \in S_{j}^{(i)}\} - \tau_{\mathrm{stab}},\\
    \tau_{\mathrm{spr}, m} &= \max_{i \in \mathcal{I}_{m}} \tau_{\mathrm{spr}}^{(i)},\\
    \tau_{\mathrm{spr}} &= \max_{m \in [M]} \tau_{\mathrm{spr}, m}, \\
    \tau &= \tau_{\mathrm{stab}} + \tau_{\mathrm{spr}}.
\end{align*}
$\tau_{\mathrm{stab}}^{(i)}$ is the earliest phase, following which if agent $i$ has their best arm in its active set, will play it most number of times and recommend it during information pulls requested by other agents.
Furthermore, starting from phase $\tau$, all the agents contain their best arms in their respective active sets, such that it will also be their most played arm and hence, will recommend it at the end of any phase $j \geq \tau$.
Mathematically, it implies the following: for all $i \in [N]$,
\begin{equation}
    \label{prop:freeze}
    k^{*}(c(i)) \in S_{j}^{(i)}, \widehat{\mathcal{O}}_{j}^{(i)} = k^{*}(c(i)) \; \forall j \geq \tau. 
\end{equation}
Claim (\ref{prop:freeze}) can be shown by induction as follows: it is evident from the definition of $\tau_{\mathrm{spr}}^{(i)}$ that $k^{*}(c(i)) \in S_{\tau_{\mathrm{stab}}+ \tau_{\mathrm{spr}}^{(i)}}^{(i)}$. Furthermore, $\widehat{\mathcal{O}}_{j}^{(i)} = k^{*}(c(i))$ for any phase $j \geq \tau_{\mathrm{stab}}$ if $k^{*}(c(i)) \in S_{j}^{(i)}$. Therefore, $\widehat{\mathcal{O}}_{\tau_{\mathrm{stab}}+ \tau_{\mathrm{spr}}^{(i)}}^{(i)} = k^{*}(c(i))$ the update step of the Algorithm \ref{alg:gosine} ensures that $k^{*}(c(i)) \in S_{\tau_{\mathrm{stab}}+ \tau_{\mathrm{spr}}^{(i)}+1}^{(i)}$, and hence claim (\ref{prop:freeze}) follows.

We now highlight the similarities and differences of the random times defined in this work with respect to the random times defined in \cite{aistats2020gossip}.
$\tau_{\mathrm{stab}}$ is defined exactly the same as in \cite{aistats2020gossip}, because we will demonstrate in Lemma \ref{prop:regfreeze} that the bound on the tail probability of the random variable $\tau_{\mathrm{stab}}^{(i)}$ is independent of the bandit played by an agent. 
However, in contrast to \cite{aistats2020gossip}, given that the set of agents are learning $M$ different bandits, the spread of $M$ best arms is non-trivial (compared to spreading a single best arm), because agents learning different bandits are communicating with each other, hence we have $M$ intertwined spreading processes occuring simultaneously.
Consequently, we simply cannot bound the spreading time $\tau_{\mathrm{spr}, m}$ for each of the $M$ best arms by the spreading time of a standard rumor spreading process, unlike \cite{aistats2020gossip}. 
This necessitates coupling the $M$ intertwined rumor spreading processes happening in our model to a set of $M$ independent fictitious noisy rumor spreading process happening on the subgraph of the agents for each bandit. 

\subsection{Intermediate Results}
We begin by providing a roadmap for the proof of Theorem \ref{thm:gosinereg} and then proving the intermediate results needed to complete it.
Claim (\ref{prop:freeze}) states that all the agents starting from phase $\tau$ contain the best arm of the bandit they are learning in their respective active sets and recommend it during information pulls.
This allows us to decompose the expected cumulative regret incurred by an agent into two parts: the regret up to phase $\tau$ and the regret after phase $\tau$.

We first show that the expected cumulative regret up to phase $\tau$ is bounded by a constant that depends only on the system parameters (number of agents, number of bandits and their respective arm gaps) and independent of the time horizon $T$ (Proposition \ref{prop:regdecomp}).
It follows from the observation that the probability of an agent not recommending their best arm and thus dropping it from their active set at the end of a phase is small and decreases as the phases progress, such that it doesn't happen infinitely often (Lemma \ref{lem:errorprob}).
This implies the existence of a random phase (defined by $\tau_{\mathrm{stab}}$), after which agents always recommend and never drop their respective best arms.
Post phase $\tau_{\mathrm{stab}}$, we characterize the time taken by the best arms of each of the $M$ bandits to spread across their respective agents (denoted by $\tau_{\mathrm{spr}, m}$).
Unlike \cite{aistats2020gossip}, we cannot bound the spreading time $\{\tau_{\mathrm{spr}, m}\}_{m \in [M]}$ of each of the $M$ best arms through a standard rumor spreading process.
This is because each agent communicates with either other agents learning the same bandit or the agents learning different bandits.
For each $m \in [M]$, $\tau_{\mathrm{spr}, m}$ is bounded by multiple layers of stochastic domination, described in the order in which they are applied below:
\begin{itemize}[leftmargin=*]
    \item first, reducing the system to the case when only one agent learning the $m^{\mathrm{th}}$ bandit is aware of their best arm $k^{*}(m)$
    \item second, lower bounding the spreading time of $k^{*}(m)$ through a fictitious noisy rumor spreading process unfolding on the subgraph of the agents learning the $m^{\mathrm{th}}$ bandit (described in the proof sketch of Theorem \ref{thm:gosinereg})
    \item third, coupling the fictitious noisy rumor spreading process in the subgraph of the agents learning the $m^{\mathrm{th}}$ bandit in the previous layer with a fictitious noiseless process (described in Appendix \ref{sec:rumorspread})
\end{itemize}
The last two layers of stochastic domination are absent in \cite{aistats2020gossip}, as all the agents are playing the same bandit.
The preceding discussion characterizes the regret up to phase $\tau$.

We subsequently show that the expected cumulative regret incurred by an agent after phase $\tau$ is bounded by the regret due to the arms in their sticky set and the regret due to the other $M-1$ best arms (Proposition \ref{prop:regdecomp}). This is a consequence of all the agents recommending their respective best arms after phase $\tau$ (Claim (\ref{prop:freeze})) and every agent communicating with agents learning other bandits. 

\begin{prop}
    \label{prop:regdecomp}
    The expected cumulative regret of any agent $i \in \mathcal{I}_{m}$ for all $m \in [M]$ after playing for $T$ time steps is bounded by:
    \begin{align*}
        \mathbb{E}[R_{T}^{(i)}] \leq \mathbb{E}[A_{\tau}] + \frac{\pi^2}{3}K + \sum_{k \in \{\widehat{\mathcal{S}}^{(i)} \cup \mathcal{B}^{(-m)}\} \backslash \{k^{*}(m)\}} \frac{4\alpha }{\Delta_{m, k}} \log T.
    \end{align*}
\end{prop}
\begin{proof}
    Fix a bandit $m \in [M]$ and an agent $i \in \mathcal{I}_{m}$. By regret decomposition principle, we know that
    \begin{align*}
        R_{T}^{(i)} &= \sum_{k \in [K]} \Delta_{m, k} \sum_{t=1}^{T} \mathbf{1}(I_{t}^{(i)}=k),\\
        &=\sum_{t=1}^{T} \sum_{k \in [K]} \Delta_{m, k} \mathbf{1}(I_{t}^{(i)}=k),\\
        &\leq A_{\tau} + \sum_{k \in [K]} \Delta_{m, k} \sum_{t=A_{\tau} +1}^{T} \mathbf{1}(I_{t}^{(i)}=k), 
    \end{align*}
    where $I_{t}^{(i)}$ is the arm played by agent $i$ at time $t$ and the last step follows from the fact that $\Delta_{m, k} \in (0, 1)$. Taking expectation on both sides, we get
    \begin{align}
    \label{eq:regdecomp}
        \mathbb{E}[R_{T}^{(i)}] &\leq \mathbb{E}[A_{\tau}] + \sum_{k \in [K]} \Delta_{m, k} \mathbb{E}\left[\sum_{t=A_{\tau} +1}^{T} \mathbf{1}(I_{t}^{(i)}=k)\right].
    \end{align}
    The first term $\mathbb{E}[A_{\tau}]$ is bounded in Proposition \ref{prop:regfreeze}. We now bound the second term. Let $X_{k,t}^{(i)} = \mathbf{1}\left(t>A_{\tau}, I_{t}^{(i)} = k, T_{k}^{(i)}(t-1) \leq \frac{4\alpha}{\Delta_{m,k}^2} \log T\right)$ and $Y_{k,t}^{(i)} = \mathbf{1}\left(t>A_{\tau}, I_{t}^{(i)} = k, T_{k}^{(i)}(t-1) > \frac{4\alpha}{\Delta_{m,k}^2} \log T\right)$. Then, the inner sum in the second term can be re-written as follows:
    \begin{align}
    \label{eq:indicatordecomp}
        \sum_{t=A_{\tau}+1}^{T} \mathbf{1}(I_{t}^{(i)} = k) &= \sum_{t=1}^{T} (X_{k,t}^{(i)} + Y_{k,t}^{(i)}). 
    \end{align}
    We first bound the sum $\sum_{t=1}^{T} \mathbb{E}[Y_{k,t}^{(i)}]$. Notice that
    \begin{align}
        \sum_{t=1}^{T} \mathbb{E}[Y_{k,t}^{(i)}] &= \sum_{t=1}^{T} \mathbb{P}\left(t>A_{\tau}, I_{t}^{(i)} = k, T_{k}^{(i)}(t-1) > \frac{4\alpha}{\Delta_{m,k}^2} \log T\right),\nonumber\\
        &\leq \sum_{t=1}^{T} \mathbb{P}\left(I_{t}^{(i)} = k, T_{k}^{(i)}(t-1) > \frac{4\alpha}{\Delta_{m,k}^2} \log T\right),\nonumber\\
        &\leq \sum_{t=1}^{T} 2t^{2-\frac{\alpha}{2}} \leq \frac{\pi^2}{3}\label{eq:indicatordecomp1},
    \end{align}
    where we substitute the classical estimate of the probability that UCB plays a sub-optimal arm using Chernoff-Hoeffding bound for $1$-subgaussian rewards and $\alpha>10$ in the last line.

    For an arm $k \in [K]$, we bound the sum $\sum_{t=1}^{T} \mathbb{E}[X_{k,t}^{(i)}]$ and complete the proof. By taking the expectation over $X_{k,t}^{(i)}$, we get
    \begin{align}
    \label{eq:indicatordecompsum2}
        \sum_{t=1}^{T} \mathbb{E}[X_{k,t}^{(i)}] &= \sum_{t=1}^{T} \mathbb{P}\left(t>A_{\tau}, I_{t}^{(i)} = k, T_{k}^{(i)}(t-1) \leq \frac{4\alpha}{\Delta_{m,k}^2} \log T\right).
    \end{align}
    For any arm $k \in [K]$, three cases are possible:
    \begin{itemize}[leftmargin=*]
    \item if $k \in \widehat{\mathcal{S}}^{(i)}$, i.e., if arm $k$ is one of the sticky sub-optimal arms, the sum in equation (\ref{eq:indicatordecompsum2}) is bounded above by $\frac{4\alpha}{\Delta_{m,k}^2} \log T$, This is because the event $I_{t}^{(i)} = k$ cannot occur more than $\frac{4\alpha}{\Delta_{m,k}^2} \log T$ times, otherwise it will contradict the event $T_{k}^{(i)}(t-1) \leq \frac{4\alpha}{\Delta_{m,k}^2} \log T$.
    \item if arm $k$ is a non-sticky arm in the active set, it has to be either the best arm $k^{*}(m)$ or one of the other best arms from the set $\mathcal{B}^{(-m)}$. This follows from Claim (\ref{prop:freeze}), as starting from phase $\tau$, agents will recommend their respective best arms during information pulls. Given that every agent communicates with agents learning other bandits, each of the arms in $\mathcal{B}^{(-m)}$ could be present in agent $i$'s active set after phase $\tau$. Therefore, for all $k \in \mathcal{B}^{(-m)}$, the sum in equation (\ref{eq:indicatordecompsum2}) is bounded by $\frac{4\alpha}{\Delta_{m,k}^2} \log T$, which follows from the same argument used for a sticky sub-optimal arm.
    \item if arm $k$ is neither a sticky sub-optimal arm nor one of the other best arms from the set $\mathcal{B}^{(-m)}$, the event $I_{t}^{(i)} = k$ cannot happen, because arm $k$ cannot be present in the active set after phase $\tau$ from Claim (\ref{prop:freeze}). Thus, the sum in equation (\ref{eq:indicatordecompsum2}) is equal to zero.
    \end{itemize}
    From the above observations, we can conclude that for all $k \in \{\widehat{\mathcal{S}}^{(i)} \cup \mathcal{B}^{(-m)}\} \backslash \{k^{*}(m)\}$, 
    \begin{align}
    \label{eq:indicatordecomp2}
        \sum_{t=1}^{T} \mathbb{E}[X_{k,t}^{(i)}] &\leq \frac{4\alpha}{\Delta_{m,k}^2} \log T.
    \end{align}
    The proof of Proposition \ref{prop:regdecomp} is completed by substituting equations (\ref{eq:indicatordecomp1}) and (\ref{eq:indicatordecomp2}) into the expectation of the equation (\ref{eq:indicatordecomp}), followed by substituting the bound on the expectation of the equation (\ref{eq:indicatordecomp}) into equation (\ref{eq:regdecomp}).
\end{proof}

In order to obtain an upper bound on $\mathbb{E}[A_{\tau}]$, we first show that the probability of the error event that the best arm is not recommended during information pull at the end of the phase $j$ (indicated by $\chi_{j}^{(i)}$ in equation (\ref{eq:errorevent})) decreases as the phases progress. This result is stated as a lemma below:

\begin{lemma}
\label{lem:errorprob}
    For any agent $i \in \mathcal{I}_{m}$ for all $m \in [M]$ and phase $j \in \mathbb{N}$ such that $\frac{A_{j}-A_{j-1}}{S+2} \geq 1 + \frac{4\alpha}{\Delta_{m}^{2}} \log A_{j}$, we have
    \begin{align*}
        \mathbb{E}[\chi_{j}^{(i)}] \leq \frac{2}{\frac{\alpha}{2}-3} {K \choose 2} (S+1) \frac{1}{A_{j-1}^{\frac{\alpha}{2}-3}}.
    \end{align*}
\end{lemma}
\begin{proof}
    The proof of this lemma is identical to the proof of Lemma 6 in \cite{aistats2020gossip}, except that it is re-written in a general form here (in \cite{aistats2020gossip}, $S := |\widehat{\mathcal{S}}^{(i)}| =  \lceil\frac{K}{N}\rceil$) and uses $1$-subgaussian rewards instead of Bernoulli rewards. 
\end{proof}

\begin{prop}
\label{prop:regfreeze}
    The regret up to phase $\tau$ is bounded by
    \begin{align*}
        \mathbb{E}[A_{\tau}] \leq \lceil (\tau^{*})^{\beta} \rceil + \frac{N (2^{\beta}+1) 2^{\beta(\frac{\alpha}{2}-3)}}{\frac{\alpha}{2}-3} {K \choose 2} (S+1) \frac{\pi^2}{3} + \sum_{m \in [M]}\mathbb{E}[A_{2\tau_{\mathrm{spr}, m}}],
    \end{align*}
    where $\tau^{*}$ is defined in Theorem \ref{thm:gosinereg}.
\end{prop}
\begin{proof}
    The random variable $\tau$ has support in $\mathbb{N}$. For $\mathbb{N}$-valued random variables, we know that
    \begin{align}
        \mathbb{E}[A_{\tau}] &= \sum_{t \geq 1} \mathbb{P}(A_{\tau} \geq t),\nonumber\\
        &\stackrel{(a)}\leq \sum_{t \geq 1} \mathbb{P}(\tau \geq A^{-1}(t)),\nonumber\\
        &\leq \sum_{t \geq 1} \mathbb{P}(\tau_{\mathrm{stab}} + \tau_{\mathrm{spr}} \geq A^{-1}(t)), \nonumber \\
        &\leq \sum_{t \geq 1} \mathbb{P}\left(\tau_{\mathrm{stab}} \geq \frac{A^{-1}(t)}{2}\right) + \sum_{t \geq 1} \mathbb{P}\left( \tau_{\mathrm{spr}} \geq \frac{A^{-1}(t)}{2}\right), \nonumber\\
        &= \sum_{t \geq 1} \mathbb{P}\left(\tau_{\mathrm{stab}} \geq \frac{A^{-1}(t)}{2}\right) + \sum_{t \geq 1} \mathbb{P}\left(\bigcup_{m=1}^{M}\left(\tau_{\mathrm{spr}, m} \geq \frac{A^{-1}(t)}{2}\right)\right),\nonumber \\
        &\leq \sum_{t \geq 1} \mathbb{P}\left(\tau_{\mathrm{stab}} \geq \frac{A^{-1}(t)}{2}\right) + \sum_{m \in [M]}\sum_{t \geq 1} \mathbb{P}\left(\tau_{\mathrm{spr}, m} \geq \frac{A^{-1}(t)}{2}\right), \nonumber\\
        &\leq A_{\tau^{*}} + \sum_{t \geq A_{\tau^{*}} + 1} \mathbb{P}\left(\tau_{\mathrm{stab}} \geq \frac{A^{-1}(t)}{2}\right) + \sum_{m \in [M]}\mathbb{E}[A_{2\tau_{\mathrm{spr}, m}}], \label{eq:regfreeze}
    \end{align}
    where we use the definition of $A^{-1}(.)$ defined in equation (\ref{eq:commepochinv}) in step $(a)$. 
    Unlike \cite{aistats2020gossip}, we cannot bound the spreading time $\{\tau_{\mathrm{spr}, m}\}_{m \in [M]}$ of each of the $M$ best arms through a standard rumor spreading process. 
    Instead, we bound $\mathbb{E}[A_{2\tau_{\mathrm{spr}, m}}]$ for all $m \in [M]$ in Appendix \ref{subsec:gosineregpfconclude} by stochastically dominating the random variable $\tau_{\mathrm{spr}, m}$ through a fictitious noisy rumor spreading process, which is described in Section \ref{pfsk:gosine} and proved in Proposition \ref{prop:noisyrumorcouple}.
    
    Here, we focus on bounding the sum $\sum_{t \geq A_{\tau^{*}} + 1} \mathbb{P}\left(\tau_{\mathrm{stab}} \geq \frac{A^{-1}(t)}{2}\right)$. We will calculate $\mathbb{P}\left(\tau_{\mathrm{stab}} \geq x \right)$ for some fixed $x \geq \frac{\tau^{*}}{2}$ using Lemma \ref{lem:errorprob} and then bound the sum in the previous sentence.
    \begin{align*}
        \mathbb{P}\left(\tau_{\mathrm{stab}} \geq x \right) &\stackrel{(a)}= \mathbb{P}\left(\bigcup_{i \in [N]} \left(\tau_{\mathrm{stab}}^{(i)} \geq x\right) \right),\\
        &\leq \sum_{i=1}^{N} \mathbb{P}\left(\tau_{\mathrm{stab}}^{(i)} \geq x \right),\\
        &\stackrel{(b)}= \sum_{i=1}^{N} \mathbb{P}\left(\bigcup_{l \geq x} \left(\chi_{l}^{(i)}=1\right) \right),\\
        &\leq \sum_{i=1}^{N} \sum_{l \geq x} \mathbb{P}\left(\chi_{l}^{(i)}=1 \right),\\
        &\stackrel{(c)} \leq \sum_{i=1}^{N} \sum_{l \geq x} \frac{2}{\frac{\alpha}{2}-3} {K \choose 2} (S+1) \frac{1}{A_{l-1}^{\frac{\alpha}{2}-3}},\\
        &\leq \sum_{l \geq x} \frac{2N}{\frac{\alpha}{2}-3} {K \choose 2} (S+1) \frac{1}{A_{l-1}^{\frac{\alpha}{2}-3}},
    \end{align*}
    where we used the definitions of $\tau_{\mathrm{stab}}$ and $\tau_{\mathrm{stab}}^{(i)}$ in the steps $(a)$ and $(b)$ respectively, and Lemma \ref{lem:errorprob} in step $(c)$ because it holds for any phase $j \geq \frac{\tau^{*}}{2}$ by definition of $\tau^{*}$. Therefore,
    \begin{align*}
        \sum_{t \geq A_{\tau^{*}} + 1} \mathbb{P}\left(\tau_{\mathrm{stab}} \geq \frac{A^{-1}(t)}{2}\right) &\leq \sum_{t \geq A_{\tau^{*}} + 1} \sum_{l \geq \frac{A^{-1}(t)}{2}} \frac{2N}{\frac{\alpha}{2}-3} {K \choose 2} (S+1) \frac{1}{A_{l-1}^{\frac{\alpha}{2}-3}},\\
        &\stackrel{(a)}\leq \sum_{l \geq \frac{\tau^{*}}{2}} \sum_{t=A_{\tau^{*}} + 1}^{A_{2l}} \frac{2N}{\frac{\alpha}{2}-3} {K \choose 2} (S+1) \frac{1}{A_{l-1}^{\frac{\alpha}{2}-3}},\\
        &\leq \frac{2N}{\frac{\alpha}{2}-3} {K \choose 2} (S+1) \sum_{l \geq \frac{\tau^{*}}{2}} \frac{A_{2l}}{A_{l-1}^{\frac{\alpha}{2}-3}},\\
        &\stackrel{(b)}= \frac{2N}{\frac{\alpha}{2}-3} {K \choose 2} (S+1) \sum_{l \geq \frac{\tau^{*}}{2}} \frac{\lceil (2l)^{\beta} \rceil}{\lceil (l-1)^{\beta} \rceil^{\frac{\alpha}{2}-3}},\\
        &\stackrel{(c)}\leq \frac{2N}{\frac{\alpha}{2}-3} {K \choose 2} (S+1) \sum_{l \geq \frac{\tau^{*}}{2}} \frac{(2l)^{\beta}+1}{(l-1)^{\beta(\frac{\alpha}{2}-3)}},\\
        &\stackrel{(d)}\leq \frac{2N}{\frac{\alpha}{2}-3} {K \choose 2} (S+1) (2^{\beta}+1) 2^{\beta(\frac{\alpha}{2}-3)}\sum_{l \geq \frac{\tau^{*}}{2}} \frac{l^{\beta}}{l^{\beta(\frac{\alpha}{2}-3)}},\\
        &\stackrel{(e)}\leq \frac{2N (2^{\beta}+1) 2^{\beta(\frac{\alpha}{2}-3)}}{\frac{\alpha}{2}-3} {K \choose 2} (S+1) \sum_{l \geq 2} \frac{1}{l^{\beta(\frac{\alpha}{2}-4)}},\\
        &\leq \frac{N (2^{\beta}+1) 2^{\beta(\frac{\alpha}{2}-3)}}{\frac{\alpha}{2}-3} {K \choose 2} (S+1) \frac{\pi^2}{3}.
    \end{align*}
    Step $(a)$ follows by re-writing the range of summations. In step $(b)$, we use $A_{j} = \lceil j^{\beta} \rceil$. Step $(c)$ uses the property of ceiling function: $x \leq \lceil x \rceil < x+1$. Steps $(d)$ and $(e)$ use the fact that $l-1 \geq \frac{l}{2}$ for all $l \geq 2$ and $\tau^{*} \geq 4$. The last step follows under the assumption that $\alpha>10$ and $\beta>2$.

    Substituting the above bound in equation (\ref{eq:regfreeze}) completes the proof of Proposition \ref{prop:regfreeze}.
\end{proof}

\begin{prop}
    \label{prop:taubound}
    $\tau_{m}^{*}$ defined in Theorem \ref{thm:gosinereg} is bounded by
    \begin{equation*}
        \tau_{m}^{*} \leq 2+\left(\frac{1}{\beta}+\left(\frac{1}{\beta}+\frac{8\alpha}{\Delta_{m}^{2}}\right)\left(S+2\right)\right)^{\frac{1}{\beta-2}}.
    \end{equation*}
\end{prop}
\begin{proof}
        From Theorem \ref{thm:gosinereg},
        \begin{equation*}
             \tau_{m}^{*} = \inf \left\{j \in \mathbb{N}: \frac{A_{j} - A_{j-1}}{S+2} \geq 1 + \frac{4\alpha}{\Delta_{m}^{2}}\log A_{j}\right\}.
        \end{equation*}
        For $A_{j} =  \lceil j^{\beta} \rceil$ and $j \geq 2+\left(\frac{1}{\beta}+ \left(\frac{1}{\beta}+\frac{8\alpha}{\Delta_{m}^{2}}\right)\left(S+2\right)\right)^{\frac{1}{\beta-2}}$,
        \begin{align*}
            1 + \frac{4\alpha}{\Delta_m^2}\log A_{j} &= 1 + \frac{4\alpha}{\Delta_m^2}\log \lceil j^{\beta} \rceil,\\
            &\stackrel{(a)}\leq 1 + \frac{4\alpha}{\Delta_m^2}\log (j^{\beta}+1),\\
            &\leq 1 + \frac{4\alpha}{\Delta_m^2}\log (2j^{\beta}),\\
            &= 1 + \frac{4\alpha}{\Delta_m^2}\log 2 + \frac{4\alpha\beta}{\Delta_m^2}\log j
            \leq 1+\frac{8\alpha\beta}{\Delta_m^2}\log j,
        \end{align*}
        where we use $\lceil x \rceil < x+1$ for any $x \in \mathbb{R}$ in step $(a)$ and $j \geq 2+\left(\frac{1}{\beta}+\left(\frac{1}{\beta}+\frac{8\alpha}{\Delta_{m}^{2}}\right)\left(S+2\right)\right)^{\frac{1}{\beta-2}} > 2$ in the last step. Therefore,
        \begin{align*}
            1 + (S+2) \left(1 + \frac{4\alpha}{\Delta_m^2}\log A_{j}\right) &\leq 1 + \left(1+\frac{8\alpha\beta}{\Delta_m^2}\log j\right)(S+2)\log j,\\
            &\leq \beta \left(\frac{j-1}{\beta} + \left(\frac{j-1}{\beta}+\frac{8\alpha}{\Delta_m^2}(j-1)\right)(S+2)\right),\\
            &\leq \beta (j-1)(j-2)^{\beta-2} 
            \leq \beta (j-1)^{\beta-1},
        \end{align*}
        by assumption on $j$ and $\log j \leq j-1$. Furthermore, 
        \begin{align*}
            A_{j} - A_{j-1} &= \lceil j^{\beta} \rceil - \lceil (j-1)^{\beta} \rceil,\\
            &\geq j^{\beta} - (j-1)^{\beta} -1,\\
            &=\beta \widehat{j}^{\beta-1} -1 \;\textrm{for some }\widehat{j} \in (j-1, j),\\
            &\geq \beta (j-1) ^{\beta-1} -1, 
        \end{align*}
        by Lagrange's Mean Value Theorem. Thus, we have shown
        $1 + \frac{4\alpha}{\Delta_m^2}\log A_{j} \leq \frac{\beta (j-1) ^{\beta-1} -1}{S+2} \leq  \frac{A_{j} - A_{j-1}}{S+2}$. This completes the proof of Proposition \ref{prop:taubound}.
    \end{proof}

\subsection{Coupling the Noisy Rumor Spreading Process with the Noiseless Process}
\label{sec:rumorspread}
\begin{prop}
    \label{prop:noisyrumorcouple}
    The random variable $\tau_{\mathrm{spr},m}$ is stochastically dominated by $\bar{\tau}_{\mathrm{spr},m}$ for all $m \in [M]$.
\end{prop}
\begin{proof}
    Follows from the construction of the fictitious noisy rumor spreading process in Section \ref{pfsk:gosine}.
\end{proof}

Let $G_{m}$ denote the gossip matrix of the subgraph of the agents learning the $m^{\mathrm{th}}$ bandit.
In our work, $G_{m}$ is a complete graph of size $N_m$, i.e., $G_{m}(i,n) = (N_m - 1)^{-1}$ for all $i \in \mathcal{I}_{m}$ and $n \in \mathcal{I}_{m} \backslash \{i\}$.
Similar to \cite{vial2022robust}, we begin by defining the variables pertinent to the fictitious noisy rumor spreading process, followed by coupling the fictitious noisy rumor spreading process with a fictitious noiseless process unfolding on $G_{m}$. The fictitious noiseless process is defined in the same way as the fictitious noisy process in Section \ref{pfsk:gosine}, but with $\eta = 1$.

\begin{definition}
    For each $i \in \mathcal{I}_{m}$, let $\{Y_{j}^{(i)}\}_{j=1}^{\infty}$ be i.i.d. Bernoulli $(\eta)$ random variables (with $\eta = \frac{N_m - 1}{N-1}$) and $\{\bar{H}_{j}^{(i)}\}_{j=1}^{\infty}$ be i.i.d. random variables chosen uniformly at random from $\mathcal{I}_{m} \backslash \{i\}$. We define $\bar{\mathcal{R}}_{m, j}$ as follows: $\bar{\mathcal{R}}_{m, 0} = \{i_{m}^{*}\}$ (from Assumption \ref{assume:stickyset}), and
    \begin{align*}
        \bar{\mathcal{R}}_{m, j} = \bar{\mathcal{R}}_{m, j-1} \cup \{i \in \mathcal{I}_{m} \backslash \bar{\mathcal{R}}_{m, j-1}: \bar{Y}_{j}^{(i)} = 1, \bar{H}_{j}^{(i)} \in \bar{\mathcal{R}}_{m, j-1}\}\;\forall\; j\in \mathbb{N}.
    \end{align*}
    Then, we can define $\tau_{\mathrm{spr},m} = \inf\{j \in \mathbb{N}: \bar{\mathcal{R}}_{m, j} = \mathcal{I}_{m}\}$.
\end{definition}

We now couple the fictitious noisy rumor spreading process introduced in the proof sketch with the fictitious noiseless process to obtain a bound on $\mathbb{E}[A_{2\bar{\tau}_{\mathrm{spr},m}}]$, which will also bound $\mathbb{E}[A_{2\tau_{\mathrm{spr},m}}]$ following Proposition \ref{prop:noisyrumorcouple}. We describe the fictitious noiseless process, which is restated from \cite{vial2022robust} for the sake of completeness.
Fix a bandit $m \in [M]$. 
Let $\{\underline{H}_{j}^{(i)}\}_{j=1}^{\infty}$ be i.i.d. Uniform $(N_{\mathrm{hon}}(i))$ random variables for each $i \in \mathcal{I}_{m}$. Note that in our setting, $N_{\mathrm{hon}}(i) = \mathcal{I}_{m} \backslash \{i\}$. Let
\begin{align*}
    \underline{\mathcal{R}}_{m,0} = \{i_{m}^{*}\}, \underline{\mathcal{R}}_{m,j} = \underline{\mathcal{R}}_{m,j-1} \cup \{i \in \mathcal{I}_{m} \backslash \underline{\mathcal{R}}_{m,j-1}: \underline{H}_{j}^{(i)} \in \underline{\mathcal{R}}_{m,j-1}\}\; \forall\; j \in \mathbb{N},
\end{align*}
and let $\underline{\tau}_{\mathrm{spr},m} = \inf\{j \in \mathbb{N}: \underline{\mathcal{R}}_{m, j} = \mathcal{I}_{m}\}$.\\

Next, we define the variables pertinent tor the coupling between the fictitious noisy and the fictitious noiseless processes.
Let
\begin{align*}
    \sigma_{0} = 0, \sigma_{l} = \inf \left\{ j > \sigma_{l-1}: \min_{i \in \mathcal{I}_{m}} \sum_{s=\sigma_{l-1}+1}^{j} \bar{Y}_{s}^{(i)} \geq 1\right\}\; \forall \; l \in \mathbb{N}.
\end{align*}
Furthermore, for each $i \in \mathcal{I}_{m}$ and $l \in \mathbb{N}$, let $Z_{l}^{(i)} = \min\{j \in \{1+\sigma_{l-1}, \cdots, \sigma_{l}\}: Y_{j}^{(i)}=1\}$.
Note that $\{Z_{l}^{(i)}\}_{i \in \mathcal{I}_{m}, l \in \mathbb{N}}$ is non-empty, and since $Z_{l}^{(i)}$ is a deterministic function of $\{\bar{Y}_{j}^{(i)}\}_{j=1}^{\infty}$ (which is independent of $\{\bar{H}_{j}^{(i)}\}_{j=1}^{\infty}$), $\{\bar{H}_{Z_{l}^{(i)}}^{(i)}\}_{j=1}^{\infty}$ is also Uniform$(\mathcal{I}_{m} \backslash \{i\})$ for each $l \in \mathbb{N}$.
Thus, we can set
\begin{gather*}
    \underline{H}_{j}^{(i)} = 
    \begin{cases}
    \bar{H}_{Z_{l}^{(i)}}^{(i)} &\text{if } j=Z_{l}^{(i)} \text{ for some } l \in \mathbb{N}\\
    \text{Uniform}(\mathcal{I}_{m} \backslash \{i\}) &\text{otherwise}
    \end{cases}
\end{gather*}
without changing the distribution of $\{\underline{\mathcal{R}}_{m,j}\}_{j=0}^{\infty}$.
This results in a coupling where the fictitious noiseless process dominates the fictitious noisy process, as follows:

\begin{prop}
\label{prop:couplingnoisynoiseless}
    (Claim 5 in \cite{vial2022robust}) For the coupling described above, $\underline{\mathcal{R}}_{m,j} \subset \bar{\mathcal{R}}_{m,\sigma_{j}}$ for all $j \in \mathbb{N}$.
\end{prop}
Proposition \ref{prop:couplingnoisynoiseless} allows us to relate the rumor spreading time of the fictitious noisy and the fictitious noiseless processes, denoted by $\bar{\tau}_{\mathrm{spr},m}$ and $\underline{\tau}_{\mathrm{spr},m}$.
In order to do so, we restate the following result from \cite{vial2022robust} in the context of our setting.

\begin{prop}
\label{prop:sprtimerelate}
    (Claim 6 in \cite{vial2022robust}) For any $j \geq 3$ and $\iota > 1$, we have $\mathbb{P}\left(\bar{\tau}_{\mathrm{spr},m} > \frac{\iota j \log j}{\eta}\right) \leq \mathbb{P}(\underline{\tau}_{\mathrm{spr},m}) + 27 c_2 \frac{N}{M}j^{1-\iota}$.
\end{prop}

We now state the result bounding $\mathbb{E}[A_{2\bar{\tau}_{\mathrm{spr},m}}]$ with $A_j = \lceil j^\beta \rceil$.

\begin{prop}
    \label{prop:sprtimebound}
    Under the conditions of Theorem \ref{thm:gosinereg}, $\mathbb{E}[A_{2\bar{\tau}_{\mathrm{spr},m}}]$ scales as 
    $O\left(\left(M\left(\log \frac{N}{M}\right)^{2} \log \left(\log \frac{N}{M}\right)\right)^{\beta}\right)$, where $O(.)$ only hides the absolute constants.
\end{prop}

We refer the interested reader to \cite{vial2022robust} (Appendix D.2) for the details about the proofs of Propositions \ref{prop:sprtimerelate} and \ref{prop:sprtimebound}. 
The results in the Appendix D.2 of \cite{vial2022robust} (Claims 7 and 8 in particular) are for $d$-regular graphs with conductance $\phi$.
We would like to point out that in our setting, for each bandit $m \in [M]$, $G_{m}$ is the gossip matrix of a complete graph of size $N_m$.
Given that a complete graph of size $N_m$ is a $d$-regular graph with $d = N_m - 1$, hence the conductance $\phi = \frac{N_m}{2(N_m -1)}$. 
Proposition \ref{prop:sprtimebound} follows by substituting the aforementioned values of $d$ and $\phi$, along with $\eta$ and using the assumption that $N_m = \Theta(\frac{N}{M})$ in the results in the Appendix D.2 of \cite{vial2022robust}.

\subsection{Completing the proof of Theorem \ref{thm:gosinereg}}
\label{subsec:gosineregpfconclude}
Propositions \ref{prop:noisyrumorcouple} and \ref{prop:sprtimebound} imply that $\mathbb{E}[A_{2\tau_{\mathrm{spr},m}}]$ also scales as $O\left(\left(M\left(\log \frac{N}{M}\right)^{2} \log \left(\log \frac{N}{M}\right)\right)^{\beta}\right)$.
Combining the above observation with Propositions \ref{prop:regfreeze} and \ref{prop:regdecomp} completes the proof of Theorem \ref{thm:gosinereg}.

\section{Random Initialization of Sticky Sets}
\label{appdx:stickysetsize}
\begin{prop}
\label{prop:bestarmstickyset}
    If $S=\left\lceil \frac{MK}{c_1 N} \log \frac{M}{\gamma} \right\rceil$ for some $\gamma \in (0, 1)$ and we construct $\widehat{S}^{(i)}$ for each agent $i \in [N]$ by sampling $S$ arms independently and uniformly at random from $K$ arms, then Assumption \ref{assume:stickyset} holds with probability at least $1-\gamma$.
\end{prop}

\begin{proof}
Fix a bandit $m \in [M]$ and let $i \in \mathcal{I}_{m}$. Then,
\begin{align*}
    \mathbb{P}(k^{*}(m) \notin \widehat{S}^{(i)}) = \frac{{K-1 \choose S}}{{K \choose S}} =\frac{K-S}{K} = 1-\frac{S}{K}.
\end{align*}
Let $\widehat{E}_{m}$ denote the event that $k^{*}(m) \notin \widehat{S}^{(i)}$ 
 for all $i \in \mathcal{I}_{m}$.Then,
\begin{align*}
     \mathbb{P}(\widehat{E}_{m}) = \mathbb{P}\left(\bigcap_{i \in \mathcal{I}_{m}} (k^{*}(m) \notin \widehat{S}^{(i)})\right) &= \prod_{i \in \mathcal{I}_{m}} \mathbb{P}(k^{*}(m) \notin \widehat{S}^{(i)}) ,\\
      &= \left(1-\frac{S}{K}\right)^{N_m} \leq e^{-\frac{N_m S}{MK}} \leq e^{-c_1\frac{NS}{MK}}.
\end{align*}
In order for Assumption \ref{assume:stickyset} to fail, there must be at least one bandit for which no agent learning that bandit has the best arm.
Therefore, Assumption \ref{assume:stickyset} fails with probability
\begin{align*}
    \mathbb{P}\left(\bigcup_{m \in [M]} \widehat{E}_{m}\right) &\leq \sum_{m \in [M]} \mathbb{P}(\widehat{E}_{m}),\\
    &\leq Me^{-c_1 \frac{NS}{MK}}.
\end{align*}
Setting $S = \left\lceil \frac{MK}{c_1 N} \log \frac{M}{\gamma} \right\rceil$ completes the proof.
\end{proof}

\section{Proof of Theorem \ref{thm:gosinesideinforeg}}
\label{pf:gosinesideinforeg}
We need additional notation for proving Theorem \ref{alg:gosinesideinfo}, particularly due to the complicated active set updates at the end of a phase.
Let $\{same(z)\}\}_{z \in [\frac{N}{r}]}$ be a partition of the set of all the agents $[N]$ consisting of $\frac{N}{r}$ sets, such that: (i) $|same(z)|  = r$ for all $z \in [\frac{N}{r}]$, and (ii) for any $i$, $i^{'} \in same(z)$ such that $i \neq i^{'}$, $c(i)=c(i^{'})$ and if $i \in f(i^{'})$, then $i^{'} \in f(i)$.
In words, for each $z$, $same(z)$ consists of agents learning the same bandit, which each of the agents in $same(z)$ are aware of, i.e., for some $z \in [\frac{N}{r}]$, $same(z) = \{i, f(i)\} = \{i \cup f(i)\}$ for some $i \in [N]$. 
Similar to the proof of Theorem \ref{thm:gosinereg}, we define some random times, which will help us prove Theorem \ref{thm:gosinesideinforeg}.
\begin{align*}
    \tau_{\mathrm{stab}}^{(i)} &= \inf\{ j \geq j^{*}: \forall l \geq j, \chi_{l}^{(i)} = 0\}, \\
    \tau_{\mathrm{stab}} &= \max_{i \in [N]} \tau_{\mathrm{stab}}^{(i)}, \\ 
    \tau_{\mathrm{spr}}^{(i)} &= \inf \{j \geq \tau_{\mathrm{stab}}: k^{*}(c(i)) \in S_{j}^{(i)}\} - \tau_{\mathrm{stab}},\\
    \tau_{\mathrm{spr}, m} &= \max_{i \in \mathcal{I}_{m}} \tau_{\mathrm{spr}}^{(i)},\\
    \tau_{\mathrm{spr}} &= \max_{m \in [M]} \tau_{\mathrm{spr}, m}, \\
    \tau_{\mathrm{rec},z} &= \inf\{j \geq \tau_{\mathrm{stab}} + \tau_{\mathrm{spr}}: \mathrm{uniquerec}(same(z),j) = \mathcal{B}\} - (\tau_{\mathrm{stab}} + \tau_{\mathrm{spr}}),\\
    \tau_{\mathrm{rec}} &= \max_{z \in [\frac{N}{r}]} \tau_{\mathrm{rec},z},\\
    \tau &= \tau_{\mathrm{stab}} + \tau_{\mathrm{spr}} + \tau_{\mathrm{rec}}.
\end{align*}

For each group of agents $z \in [\frac{N}{r}]$ learning the same bandit and are aware of that, $\tau_{\mathrm{rec},z}$ denotes the minimum number of phases it takes after the phase $\tau_{\mathrm{stab}} + \tau_{\mathrm{spr}}$ such that the $M$ most recent unique arm recommendations among all the agents in $same(z)$ is the set of $M$ best arms.
The active set updates in Algorithm \ref{alg:gosinesideinfo} ensure that the active sets of agents freeze after phase $\tau$ (like \cite{aistats2020gossip}), unlike Algorithm \ref{alg:gosine}, where the active sets are time-varying despite agents eventually identifying their respective best arms and recommending it in information pulls. We now prove this claim.

\begin{prop}
\label{prop:sideinfofreeze}
For any agent $i \in \mathcal{I}_{m}$ playing Algorithm \ref{alg:gosinesideinfo}, the following statements hold for all $j > \tau$:

(i) $k^{*}(m) \in S_{j}^{(i)}$,

(ii) $S_{j}^{(i)} = S_{\tau}^{(i)}$.

(iii) $\left\{\widetilde{S}_{\tau}^{(i)} \cup \{\mathcal{O}_{\tau}^{(i)}\}\right\} \subset \mathcal{B}$.
\end{prop}

\begin{proof}
(i) follows exactly from Claim (\ref{prop:freeze}).

For proving (ii), notice that for any group of agents $same(z)$ and any phase $j \geq \tau_{\mathrm{stab}} + \tau_{\mathrm{spr}} + \tau_{\mathrm{rec},z}$, $\mathrm{uniquerec}(same(z),j) = \mathcal{B}$.
This follows from Claim (\ref{prop:freeze}) because after the phase $\tau_{\mathrm{stab}} + \tau_{\mathrm{spr}}$, agents recommend their respective best arms during information pulls and eventually, the $M$ most recent unique arm recommendations for all the agents in $same(z)$ (denoted by $\mathrm{uniquerec}(same(z),j-1)$) will be the set of $M$ best arms.
Furthermore, the active set updates in Algorithm \ref{alg:gosinesideinfo} ensure that if $\mathrm{uniquerec}(same(z),j) = \mathrm{uniquerec}(same(z),j-1)$, the active sets for all the agents in $same(z)$ remain unchanged.
Since $\tau \geq \tau_{\mathrm{stab}} + \tau_{\mathrm{spr}} + \tau_{\mathrm{rec},z}$, claim (ii) holds.

Claim (iii) follows from the observation that for any group of agents $same(z)$ and any phase $j \geq \tau_{\mathrm{stab}} + \tau_{\mathrm{spr}} + \tau_{\mathrm{rec},z}$, $\mathrm{uniquerec}(same(z),j) = \mathcal{B}$, $\widetilde{S}_{j}^{(i)} \subset \mathrm{uniquerec}(same(z),j)$ by construction and $\mathcal{O}_{j}^{(i)} \in \mathcal{B}$ after any phase $j \geq \tau_{\mathrm{stab}} + \tau_{\mathrm{spr}}$.
\end{proof}

\subsection{Intermediate Results}
Most of the intermediate results for Algorithm \ref{alg:gosinesideinfo} are similar to the results for Algorithm \ref{alg:gosine}, and are proved in a similar way.
The technical novelty in proving Theorem \ref{thm:gosinesideinforeg} is to prove that $\mathbb{E}[\tau_{\mathrm{rec}}]$ is finite.

Before decomposing the regret up to phase $\tau$ and after phase $\tau$, we define some notation. For any bandit $m \in [M]$, Let $k_{m,1}, k_{m,2}, \cdots, k_{m,K}$ denote the IDs of the arms with their arm means sorted in decreasing order, i.e., $k_{m,1} = k^{*}(m)$ and $\mu_{m,k_{m,1}} > \mu_{m,k_{m,2}} \geq \cdots \geq \mu_{m, k_{m,K}}$.

\begin{prop}
    \label{prop:regdecompsideinfo}
    The expected cumulative regret of any agent $i \in \mathcal{I}_{m}$ for all $m \in [M]$ after playing for $T$ time steps is bounded by:
    \begin{align*}
        \mathbb{E}[R_{T}^{(i)}] \leq \mathbb{E}[A_{\tau}] + \frac{\pi^2}{3}K + \sum_{k \in \{k_{m,l}\}_{l=2}^{S+\lceil\frac{M}{r}\rceil+2}}  \frac{4\alpha }{\Delta_{m, k}} \log T.
    \end{align*}
\end{prop}

\begin{proof}
    Fix a bandit $m \in [M]$ and an agent $i \in \mathcal{I}_{m}$. By regret decomposition principle, we know that
    \begin{align*}
        R_{T}^{(i)} &= \sum_{k \in [K]} \Delta_{m, k} \sum_{t=1}^{T} \mathbf{1}(I_{t}^{(i)}=k),\\
        &=\sum_{t=1}^{T} \sum_{k \in [K]} \Delta_{m, k} \mathbf{1}(I_{t}^{(i)}=k),\\
        &\leq A_{\tau} + \sum_{k \in [K]} \Delta_{m, k} \sum_{t=A_{\tau} +1}^{T} \mathbf{1}(I_{t}^{(i)}=k), 
    \end{align*}
    where $I_{t}^{(i)}$ is the arm played by agent $i$ at time $t$ and the last step follows from the fact that $\Delta_{m, k} \in (0, 1)$. Taking expectation on both sides, we get
    \begin{align}
    \label{eq:regdecompsideinfo}
        \mathbb{E}[R_{T}^{(i)}] &\leq \mathbb{E}[A_{\tau}] + \sum_{k \in [K]} \Delta_{m, k} \mathbb{E}\left[\sum_{t=A_{\tau} +1}^{T} \mathbf{1}(I_{t}^{(i)}=k)\right].
    \end{align}
    The first term $\mathbb{E}[A_{\tau}]$ is bounded in Proposition \ref{prop:regsideinfofreeze}. We now bound the second term. Let $X_{k,t}^{(i)} = \mathbf{1}\left(t>A_{\tau}, I_{t}^{(i)} = k, T_{k}^{(i)}(t-1) \leq \frac{4\alpha}{\Delta_{m,k}^2} \log T\right)$ and $Y_{k,t}^{(i)} = \mathbf{1}\left(t>A_{\tau}, I_{t}^{(i)} = k, T_{k}^{(i)}(t-1) > \frac{4\alpha}{\Delta_{m,k}^2} \log T\right)$. Then, the inner sum in the second term can be re-written as follows:
    \begin{align}
    \label{eq:indicatordecompsideinfo}
        \sum_{t=A_{\tau}+1}^{T} \mathbf{1}(I_{t}^{(i)} = k) &= \sum_{t=1}^{T} (X_{k,t}^{(i)} + Y_{k,t}^{(i)}). 
    \end{align}
    We first bound the sum $\sum_{t=1}^{T} \mathbb{E}[Y_{k,t}^{(i)}]$. Notice that
    \begin{align}
        \sum_{t=1}^{T} \mathbb{E}[Y_{k,t}^{(i)}] &= \sum_{t=1}^{T} \mathbb{P}\left(t>A_{\tau}, I_{t}^{(i)} = k, T_{k}^{(i)}(t-1) > \frac{4\alpha}{\Delta_{m,k}^2} \log T\right),\nonumber\\
        &\leq \sum_{t=1}^{T} \mathbb{P}\left(I_{t}^{(i)} = k, T_{k}^{(i)}(t-1) > \frac{4\alpha}{\Delta_{m,k}^2} \log T\right),\nonumber\\
        &\leq \sum_{t=1}^{T} 2t^{2-\frac{\alpha}{2}} \leq \frac{\pi^2}{3}\label{eq:indicatordecompsideinfo1},
    \end{align}
    where we substitute the classical estimate of the probability that UCB plays a sub-optimal arm using Chernoff-Hoeffding bound for $1$-subgaussian rewards and $\alpha>10$ in the last line.

    For an arm $k \in [K]$, we bound the sum $\sum_{t=1}^{T} \mathbb{E}[X_{k,t}^{(i)}]$ and complete the proof. 
    \begin{align}
    \label{eq:indicatordecompsideinfosum2}
        \sum_{t=1}^{T} X_{k,t}^{(i)} = \sum_{t=1}^{T} \mathbf{1}\left(t>A_{\tau}, I_{t}^{(i)} = k, T_{k}^{(i)}(t-1) \leq \frac{4\alpha}{\Delta_{m,k}^2} \log T\right) 
    \end{align}
    For any arm $k \in [K]$, two cases are possible:
    \begin{itemize}[leftmargin=*]
    \item if $k \in S_{\tau}^{(i)}$ and arm $k$ is one of the sub-optimal arms, the sum in equation (\ref{eq:indicatordecompsideinfosum2}) is bounded above by $\frac{4\alpha}{\Delta_{m,k}^2} \log T$, This is because the event $I_{t}^{(i)} = k$ cannot occur more than $\frac{4\alpha}{\Delta_{m,k}^2} \log T$ times, otherwise it will contradict the event $T_{k}^{(i)}(t-1) \leq \frac{4\alpha}{\Delta_{m,k}^2} \log T$.
    \item if arm $k \notin S_{\tau}^{(i)}$, the event $I_{t}^{(i)} = k$ cannot happen. Thus, the sum in equation (\ref{eq:indicatordecompsideinfosum2}) is equal to zero.
    \end{itemize}
    From the above observations, we can conclude that for all $k \in \{S_{\tau}^{(i)}\} \backslash \{k^{*}(m)\}$, 
    \begin{align}
    \label{eq:indicatordecompsideinfo2}
        \sum_{t=1}^{T} X_{k,t}^{(i)} &\leq \frac{4\alpha}{\Delta_{m,k}^2} \log T.
    \end{align}
    Therefore,
    \begin{align*}
        \sum_{k=1}^{K} \Delta_{m,k}\sum_{t=1}^{T} X_{k,t}^{(i)} &\leq \sum_{k \in S_{\tau}^{(i)}} \frac{4\alpha}{\Delta_{m,k}} \log T,\\
        &\leq \sum_{k \in \widehat{S}^{(i)} \backslash \{k^{*}(m)\}} \frac{4\alpha}{\Delta_{m,k}} \log T + \sum_{k \in \{k_{m,l}\}_{l=2}^{\lceil\frac{M}{r}\rceil+2}} \frac{4\alpha}{\Delta_{m,k}} \log T,
    \end{align*}
    because $S_{\tau}^{(i)} = \widehat{\mathcal{S}}^{(i)} \cup \{\widehat{\mathcal{O}}_{\tau}^{(i)}\} \cup \{\mathcal{O}_{\tau}^{(i)}\} \cup \{\widetilde{S}_{\tau}^{(i)}\}$, $\widehat{\mathcal{O}}_{\tau}^{(i)} = k^{*}(m)$ and we don't have any control over which arms from the set $\mathcal{B}$ are present in the set $\widetilde{S}_{\tau}^{(i)} \cup \mathcal{O}_{\tau}^{(i)}$ (Proposition \ref{prop:sideinfofreeze}), so we assume the worst case scenario and bound the regret of the arms in that set by the regret of the first $\lceil \frac{M}{r} \rceil + 2$ arms in the increasing order of their arm gaps (or equivalently, decreasing order of the arm means). By taking the expectation on both sides in the above inequality, we get
    \begin{align}
    \label{eq:indicatordecompsideinfosum3}
        \sum_{t=1}^{T} \mathbb{E}[X_{k,t}^{(i)}] &\leq \sum_{k \in \widehat{S}^{(i)} \backslash \{k^{*}(m)\}} \frac{4\alpha}{\Delta_{m,k}} \log T + \sum_{k \in \{k_{m,l}\}_{l=2}^{\lceil\frac{M}{r}\rceil+2}} \frac{4\alpha}{\Delta_{m,k}} \log T.
    \end{align}
    The proof of Proposition \ref{prop:regdecomp} is completed by substituting equations (\ref{eq:indicatordecompsideinfo1}) and (\ref{eq:indicatordecompsideinfosum3}) into the expectation of the equation (\ref{eq:indicatordecompsideinfo}), followed by substituting the bound on the expectation of the equation (\ref{eq:indicatordecompsideinfo}) into equation (\ref{eq:regdecompsideinfo}).
\end{proof}

In order to obtain an upper bound on $\mathbb{E}[A_{\tau}]$, we first show that the probability of the error event that the best arm is not recommended during information pull at the end of the phase $j$ (indicated by $\chi_{j}^{(i)}$ in equation (\ref{eq:errorevent})) decreases as the phases progress. This result is stated as a lemma below:

\begin{lemma}
\label{lem:errorprobsideinfo}
    For any agent $i \in \mathcal{I}_{m}$ for all $m \in [M]$ and phase $j \in \mathbb{N}$ such that $\frac{A_{j}-A_{j-1}}{S+\lceil \frac{M}{r} \rceil+2} \geq 1 + \frac{4\alpha}{\Delta_{m}^{2}} \log A_{j}$, we have
    \begin{align*}
        \mathbb{E}[\chi_{j}^{(i)}] \leq \frac{2}{\frac{\alpha}{2}-3} {K \choose 2+\left \lceil \frac{M}{r} \right \rceil} \left(S+\left\lceil \frac{M}{r} \right\rceil+1\right) \frac{1}{A_{j-1}^{\frac{\alpha}{2}-3}}.
    \end{align*}
\end{lemma}
\begin{proof}
    The proof of this lemma is identical to the proof of Lemma 6 in \cite{aistats2020gossip} and Lemma \ref{lem:errorprob}, except that $|S_{j}^{(i)}| \leq S+\lceil \frac{M}{r} \rceil+2$. 
\end{proof}

\begin{prop}
\label{prop:regsideinfofreeze}
    The regret up to phase $\tau$ is bounded by
    \begin{align*}
        \mathbb{E}[A_{\tau}] \leq \lceil (j^{*})^{\beta} \rceil + \frac{N (3^{\beta}+1) 3^{\beta(\frac{\alpha}{2}-3)}}{\frac{\alpha}{2}-3} {K \choose 2+\left \lceil \frac{M}{r} \right \rceil} \left(S+\left\lceil \frac{M}{r} \right\rceil+1\right) \frac{\pi^2}{3} + \sum_{m \in [M]}\mathbb{E}[A_{3\tau_{\mathrm{spr}, m}}] + \sum_{z \in [\frac{N}{r}]}\mathbb{E}[A_{3\tau_{\mathrm{rec}, z}}],
    \end{align*}
    where $j^{*}$ is defined in Theorem \ref{thm:gosinesideinforeg}.
\end{prop}
\begin{proof}
    The random variable $\tau$ has support in $\mathbb{N}$. For $\mathbb{N}$-valued random variables, we know that
    \begin{align}
        \mathbb{E[A_{\tau}]} &= \sum_{t \geq 1} \mathbb{P}(A_{\tau} \geq t),\nonumber\\
        &\stackrel{(a)}\leq \sum_{t \geq 1} \mathbb{P}(\tau \geq A^{-1}(t)),\nonumber\\
        &\leq \sum_{t \geq 1} \mathbb{P}(\tau_{\mathrm{stab}} + \tau_{\mathrm{spr}} + \tau_{\mathrm{rec}} \geq A^{-1}(t)), \nonumber \\
        &\leq \sum_{t \geq 1} \mathbb{P}\left(\tau_{\mathrm{stab}} \geq \frac{A^{-1}(t)}{3}\right) + \sum_{t \geq 1} \mathbb{P}\left( \tau_{\mathrm{spr}} \geq \frac{A^{-1}(t)}{3}\right) + \sum_{t \geq 1} \mathbb{P}\left( \tau_{\mathrm{rec}} \geq \frac{A^{-1}(t)}{3}\right), \nonumber\\
        &= \sum_{t \geq 1} \mathbb{P}\left(\tau_{\mathrm{stab}} \geq \frac{A^{-1}(t)}{3}\right) + \sum_{t \geq 1} \mathbb{P}\left(\bigcup_{m=1}^{M}\left(\tau_{\mathrm{spr}, m} \geq \frac{A^{-1}(t)}{3}\right)\right) + \sum_{t \geq 1} \mathbb{P}\left(\bigcup_{z \in [\frac{N}{r}]}\left(\tau_{\mathrm{rec}, z} \geq \frac{A^{-1}(t)}{3}\right)\right),\nonumber \\
        &\leq \sum_{t \geq 1} \mathbb{P}\left(\tau_{\mathrm{stab}} \geq \frac{A^{-1}(t)}{3}\right) + \sum_{m \in [M]}\sum_{t \geq 1} \mathbb{P}\left(\tau_{\mathrm{spr}, m} \geq \frac{A^{-1}(t)}{3}\right) + \sum_{z \in [\frac{N}{r}]}\sum_{t \geq 1} \mathbb{P}\left(\tau_{\mathrm{rec}, m} \geq \frac{A^{-1}(t)}{3}\right), \nonumber\\
        &\leq A_{j^{*}} + \sum_{t \geq A_{j^{*}} + 1} \mathbb{P}\left(\tau_{\mathrm{stab}} \geq \frac{A^{-1}(t)}{3}\right) + \sum_{m \in [M]}\mathbb{E}[A_{3\tau_{\mathrm{spr}, m}}] + \sum_{z \in [\frac{N}{r}]}\mathbb{E}[A_{3\tau_{\mathrm{rec}, z}}], \label{eq:regsideinfofreeze}
    \end{align}
    where we use the definition of $A^{-1}(.)$ defined in equation (\ref{eq:commepochinv}) in step $(a)$. 
   $\mathbb{E}[A_{3\tau_{\mathrm{spr}, m}}]$ is bounded for all $m \in [M]$ in the same way as for Algorithm \ref{alg:gosine} and is stated in Proposition \ref{prop:sprtimesideinfobound}. We bound $\mathbb{E}[A_{3\tau_{\mathrm{rec}, z}}]$ for each $z \in [\frac{N}{r}]$ in Proposition \ref{prop:taurecsideinfobound}.
    
    The process for bounding the sum $\sum_{t \geq A_{j^{*}} + 1} \mathbb{P}\left(\tau_{\mathrm{stab}} \geq \frac{A^{-1}(t)}{3}\right)$ is identical to bounding the sum $\sum_{t \geq A_{\tau^{*}} + 1} \mathbb{P}\left(\tau_{\mathrm{stab}} \geq \frac{A^{-1}(t)}{2}\right)$ for Algorithm \ref{alg:gosine} in Proposition \ref{prop:regfreeze}. 
\end{proof}

\begin{prop}
    \label{prop:tausideinfobound}
    $j_{m}^{*}$ defined in Theorem \ref{thm:gosinesideinforeg} is bounded by
    \begin{equation*}
        j_{m}^{*} \leq 2+\left(\frac{1}{\beta}+\left(\frac{1}{\beta}+\frac{8\alpha}{\Delta_{m}^{2}}\right)\left(S+2+\left\lceil \frac{M}{r} \right\rceil\right)\right)^{\frac{1}{\beta-2}}.
    \end{equation*}
\end{prop}
\begin{proof}
        The proof of Proposition \ref{prop:tausideinfobound} is identical to the proof of Proposition \ref{prop:taubound}, except that $S+2$ is replaced by $S+2+\lceil\frac{M}{r}\rceil$.
    \end{proof}

\begin{prop}
    \label{prop:sprtimesideinfobound}
    Under the conditions of Theorem \ref{thm:gosinesideinforeg}, $\mathbb{E}[A_{3\tau_{\mathrm{spr},m}}]$ scales as 
    $O\left(\left(M\left(\log \frac{N}{M}\right)^{2} \log \left(\log \frac{N}{M}\right)\right)^{\beta}\right)$, where $O(.)$ only hides the absolute constants.
\end{prop}
\begin{proof}
    The proof is identical to the proof of $\mathbb{E}[A_{2\tau_{\mathrm{spr},m}}]$ for Algorithm \ref{alg:gosine}.
\end{proof}

\begin{prop}
    \label{prop:taurecsideinfobound}
    For each $z \in [\frac{N}{r}]$, $\mathbb{E}[A_{3\tau_{\mathrm{rec},z}}]$ is bounded by
    \begin{align*}
        \mathbb{E}[A_{3\tau_{\mathrm{rec},z}}] \leq \lceil (3M)^{\beta} \rceil + 2 \left(\frac{3}{\left(\frac{c_1}{M}-\frac{1}{N}\right)r}\right)^{\beta} \frac{M}{\left(1-\frac{c_1}{M}\right)^{r}} \Gamma(\beta+1),
    \end{align*}
    where $\Gamma(\alpha) = \int_{t=0}^{\infty} t^{\alpha-1} e^{-t}\; \mathrm{d}t$ for any $\alpha > 0$ denotes the Gamma function.
\end{prop}
\begin{proof}
Fix a $z \in [\frac{N}{r}]$. Then, we know from the definition of $same(z)$ that there exists an $i \in \mathcal{I}_{m}$ such that $same(z) = \{i, f(i)\} = \{\{i\} \cup f(i)\}$ for some $m \in [M]$.

Let $E_{m}$ denote the event that an agent $i \in [N]$ doesn't contact an agent from $\mathcal{I}_{m}$ during an information pull at the end of a phase.
Then, from the choice of gossip matrix $G$ in Theorem \ref{thm:gosinesideinforeg},
\begin{gather*}
    \mathbb{P}(E_{m}) = 
    \begin{cases}
    1-\frac{N_m -1}{N-1} &\text{if } c(i)=m,\\
    1-\frac{N_m}{N-1} &\text{otherwise}.
    \end{cases}
\end{gather*}

Furthermore,
\begin{gather}
\label{eq:contactprob}
    \mathbb{P}(E_{m}) \leq 
    \begin{cases}
    1-\frac{c_1}{M} + \frac{1}{N} &\text{if } c(i)=m,\\
    1-\frac{c_1}{M} &\text{otherwise},
    \end{cases}
\end{gather}
by assumption on $N_m$ for all $m \in [M]$.

For some $m \in [M]$ and $i \in \mathcal{I}_{m}$, $same(z) = \{i, f(i)\} = \{\{i\} \cup f(i)\}$ for some $z \in [\frac{N}{r}]$. 
We also know that in $\tau_{\mathrm{rec},z}$ steps, agents in $same(z)$ will receive a total of $r\tau_{\mathrm{rec},z}$ arm recommendations, which for each agent $i \in same(z)$ are i.i.d. uniform$\{[N] \backslash \{i\}\}$ and independent across all $i \in same(z)$.
By the definition of $\tau_{\mathrm{rec},z}$, $(\tau_{\mathrm{rec},z} > x)$ denotes the event that agents in the set $same(z)$ don't have the $M$ most recent unique arm recommendations equal to the set of $M$ best arms in $x$ steps (denoted by $\mathcal{B}$).
Therefore,
\begin{align}
    \mathbb{P}(\tau_{\mathrm{rec},z} > x) &= \mathbb{P}\left(\cup_{m^{'} \in [M]} (E_{m^{'}} \text{ occurs } rx \text{ times})\right),\nonumber\\
    &\leq \sum_{m^{'} \in [M]} \mathbb{P}\left(E_{m^{'}} \text{ occurs } rx \text{ times}\right), \nonumber\\
    &\leq \left(\sum_{m^{'} \neq m} \left(1-\frac{c_1}{M}\right)^{rx}\right) + \left(1-\frac{c_1}{M}+\frac{1}{N}\right)^{rx}, \nonumber\\
    &\leq (M-1)\left(1-\frac{c_1}{M}\right)^{rx} + \left(1-\frac{c_1}{M}+\frac{1}{N}\right)^{rx}, \label{eq:taurectailprob}
\end{align}
where the last two steps follow from equation (\ref{eq:contactprob}) and the fact that the information pulls are independent across agents and phases.

We are now ready to bound $\mathbb{E}[A_{3\tau_{\mathrm{rec},z}}]$. Given that $A_{3\tau_{\mathrm{rec},z}}$ is a $\mathbb{N}$-valued random variable, we have
\begin{align}
    \mathbb{E}[A_{3\tau_{\mathrm{rec},z}}] &=\sum_{t \geq 1}\mathbb{P}(A_{3\tau_{\mathrm{rec},z}} \geq t), \nonumber\\
    &\leq A_{3M} + \sum _{t \geq A_{3M}+1}\mathbb{P}(A_{3\tau_{\mathrm{rec},z}} \geq t), \nonumber\\
    &\stackrel{(a)}\leq A_{3M} + \sum _{j \geq 1}\mathbb{P}(A_{3\tau_{\mathrm{rec},z}} \geq A_{3(M+j-1)}) (A_{3(M+j)}-A_{3(M+j-1)}), \nonumber\\
    &\stackrel{(b)}\leq A_{3M} + \sum _{j \geq 1}\mathbb{P}(\tau_{\mathrm{rec},z} \geq M+j-1) A_{3(M+j)}, \nonumber\\
    &\stackrel{(c)}\leq A_{3M} + (M-1)\sum _{j \geq 1}\left(1-\frac{c_1}{M}\right)^{r(M+j-1)} A_{3(M+j)} + \sum _{j \geq 1}\left(1-\frac{c_1}{M}+\frac{1}{N}\right)^{r(M+j-1)} A_{3(M+j)}, \label{eq:taurecboundint}
\end{align}
where step $(a)$ follows from the fact that for a random variable $X$, $\mathbb{P}(X \geq x)$ is non-increasing in $x$, step $(b)$ follows from the definition of $A^{-1}(.)$ in equation (\ref{eq:commepochinv}) and we subsitute equation (\ref{eq:taurectailprob}) in step $(c)$. 
We bound each of the two sums in equation (\ref{eq:taurecboundint}) to complete the proof.
For any $\epsilon \in (0,1)$ and $A_j = \lceil j^\beta \rceil$,
\begin{align*}
    \sum_{j \geq 1} (1-\epsilon)^{r(M+j-1)} A_{3(M+j)}&= \sum_{j \geq 1} (1-\epsilon)^{r(M+j-1)} \lceil (3(M+j))^{\beta} \rceil, \\
    &\stackrel{(a)}\leq 2 (3^{\beta}) \sum_{j \geq 1} (1-\epsilon)^{r(M+j-1)}  (M+j)^{\beta}, \\
    &= 2 (3^{\beta}) \sum_{l \geq M+1} (1-\epsilon)^{r(l-1)}  l^{\beta}, \\
    &\leq \frac{2 (3^{\beta})}{(1-\epsilon)^{r}} \sum_{l \geq M+1} e^{-\epsilon rl}  l^{\beta}, \\
    &\leq \frac{2 (3^{\beta})}{(1-\epsilon)^{r}} \int_{0}^{\infty} e^{-\epsilon ry}  y^{\beta}\;\mathrm{d}y, \\
    &\stackrel{(b)}= \frac{2}{(1-\epsilon)^{r}} \left(\frac{3}{\epsilon r}\right)^{\beta} \int_{0}^{\infty} u^{\beta} e^{-\epsilon u} \;\mathrm{d}u,
    =\frac{2}{(1-\epsilon)^{r}} \left(\frac{3}{\epsilon r}\right)^{\beta} \Gamma(\beta+1),
\end{align*}
where we use the property that $\lceil x \rceil \leq 2x$ for all $x \geq 1$ in step $(a)$ and we perform a change of variables $u=\epsilon r y$ in step $(b)$.
Substituting the above bound in equation (\ref{eq:taurecboundint}) with $\epsilon = \frac{c_1}{M}$ and $\epsilon = \frac{c_1}{M} - \frac{1}{N}$ in each of the two sums respectively, we get
\begin{align*}
    \mathbb{E}[A_{3\tau_{\mathrm{rec},z}}] &\leq A_{3M} + \frac{2(M-1)}{\left(1-\frac{c_1}{M}\right)^{r}} \left(\frac{3}{\frac{c_1}{M} r}\right)^{\beta} \Gamma(\beta+1) + \frac{2}{\left(1-\frac{c_1}{M}+\frac{1}{N}\right)^{r}} \left(\frac{3}{\left(\frac{c_1}{M}-\frac{1}{N}\right) r}\right)^{\beta} \Gamma(\beta+1),\\
    &\leq A_{3M} + \frac{2(M-1)}{\left(1-\frac{c_1}{M}\right)^{r}} \left(\frac{3}{\left(\frac{c_1}{M}-\frac{1}{N}\right)   r}\right)^{\beta} \Gamma(\beta+1) + \frac{2}{\left(1-\frac{c_1}{M}\right)^{r}} \left(\frac{3}{\left(\frac{c_1}{M}-\frac{1}{N}\right) r}\right)^{\beta} \Gamma(\beta+1),\\
    &= \lceil (3M)^{\beta} \rceil + 2 \left(\frac{3}{\left(\frac{c_1}{M}-\frac{1}{N}\right)r}\right)^{\beta} \frac{M}{\left(1-\frac{c_1}{M}\right)^{r}} \Gamma(\beta+1),
\end{align*}
which completes the Proof of Proposition \ref{prop:taurecsideinfobound}.
\end{proof}

\subsection{Completing the proof of Theorem \ref{thm:gosinesideinforeg}}
The proof of Theorem \ref{alg:gosinesideinfo} follows from the Propositions \ref{prop:regdecompsideinfo}, \ref{prop:regsideinfofreeze}, \ref{prop:sprtimesideinfobound} and \ref{prop:taurecsideinfobound}.

\section{Proofs of Lower Bounds}
\label{pf:lowerbound}
\subsection{Proof of Theorem \ref{thm:lowerbound}}

We derive the lower bound for our model by adapting the lower bound for the setting considered in \cite{reda2022near}.
\cite{reda2022near} considers the following problem: there are $N$ agents and $K$ arms. When agent $n \in [N]$ pulls arm $k \in [K]$, it observes a noisy reward with mean $\bar{\mu}_{k,n}$. However, regret is measured with respect to a ``mixed reward'' with mean $\mu_{k,n}' = \sum_{i=1}^N w_{i,n} \bar{\mu}_{k,i}$, where $\{ w_{i,n} \}_{i=1}^N$ are (known) nonnegative weights with $\sum_{i=1}^N w_{i,n} = 1$. So the optimal arm for agent $n$ is $k_n^\star = \arg \max_{k \in [K]} \mu_{k,n}'$ and the relevant arm gaps are $\Delta_{k,n}' = \mu_{k_n^\star,n}' - \mu_{k,n}'$. 
Unlike our work, \cite{reda2022near} doesn't have any constraints on the lengths of the messages exchanged in the communication rounds.

{\bf Lower bound:} Theorem 3 of \cite{reda2022near} bounds the group regret $\mathrm{Reg}(T)$ (i.e., regret summed across agents) as follows: for uniformly efficient algorithms, assuming Gaussian rewards with unit variance,
\begin{gather}
\liminf_{T \rightarrow \infty} \frac{\mathrm{Reg}(T)}{\log T}  \geq  \min_{x \in \mathcal{X}} f(x) , \quad \text{where} \\
\mathcal{X} = \left\{ x \in \mathbb{R}_+^{K \times N} : \sum_{i : k_i^\star \neq k} \frac{ w_{i,n}^2 }{ x_{k,i} } \leq \frac{ ( \Delta_{k,n}' )^2 }{2}\ \forall\ n \in [N] , k \in [K]  \right\} , \quad f(x) = \sum_{k=1}^K \sum_{n : k_n^\star \neq k} x_{k,n} \Delta_{k,n}'\ \forall\ x \in \mathcal{X}.
\end{gather}

\subsection*{Applying to our setting and proving Theorem \ref{thm:lowerbound}}

{\bf Notation:} In our setting, $c(n) \in [M]$ denotes the bandit that agent $n \in [N]$ is learning and let $c^{-1}(m) = \{ n \in [N] : c(n) = m \}$ denote the set of agents learning the bandit $m \in [M]$. 
Notice that for any agent $n \in [N]$, $c^{-1}(c(n))$ is the set of agents learning the same bandit as agent $n$ (also, $n \in c^{-1}(c(n))$). 

{\bf Reduction to the setting of \cite{reda2022near}:} For each agent $n \in [N]$, we can choose $\bar{\mu}_{k,n} = \mu_{c(n),k}$ and $w_{i,n} = \mathbf{1}(i \in c^{-1}(c(n))) / |c^{-1}(c(n))|$ for the arm means and weights. 
Then defining $\mu_{k,n}'$ as above, a simple calculation shows $\mu_{k,n}' = \mu_{c(n),k} = \bar{\mu}_{k,n}$, so $\Delta_{k,n}' = \Delta_{c(n),k}$. 
Note in particular that $\mu_{k,n}' = \bar{\mu}_{k,n}$ means the observed and mixed rewards are the same, as in our model, and thus, $k_{n}^{\star} = k^{*}(c(n))$.
Given that agents are aware of the weights in the setting of \cite{reda2022near}, in our model this corresponds to the case when every agent knows all the other agents learning the same bandit.
Additionally, in \cite{reda2022near}, there are no constraints on the length of the messages exchanged in the communication rounds.
Due to the absence of constraints on the lengths of the messages exchanged during communications, our model is a special case of the model in \cite{reda2022near}, and thus, their lower bound is applicable to our setting. 

{\bf Applying their lower bound:} In our special case, for any agent $n$ and arm $k \neq k_n^\star$, we have
\begin{equation} \label{eqSumOverI}
\sum_{i : k_i^\star \neq k} \frac{ w_{i,n}^2 }{ x_{k,i} } = \frac{1}{|c^{-1}(c(n))|^2} \sum_{ i \in c^{-1}(c(n)) : k_i^\star \neq k } \frac{1}{x_{k,i}} =  \frac{1}{|c^{-1}(c(n))|^2} \sum_{i \in c^{-1}(c(n)) } \frac{1}{x_{k,i}} ,
\end{equation}
where the first equality uses our choice of $w_{i,n}$ and the second holds since $k_i^\star = k_n^\star \neq k$ for each $i \in c^{-1}(c(n))$. Therefore, the constraint for such $n, k$ pairs simplifies to
\begin{equation*}
\frac{1}{|c^{-1}(c(n))|^2} \sum_{i \in c^{-1}(c(n)) } \frac{1}{x_{k,i}} \leq \frac{ ( \Delta_{k,n}' )^2 }{2} = \frac{\Delta_{c(n),k}^2}{2}  ,
\end{equation*}
which can be rearranged to obtain
\begin{equation*}
\frac{ 2 }{ |c^{-1}(c(n))|\Delta_{c(n),k}^2 } \leq \frac{ |c^{-1}(c(n))| }{ \sum_{i \in c^{-1}(c(n)) } \frac{1}{x_{k,i}} } = \text{HM} ( \{ x_{k,i} \}_{i \in c^{-1}(c(n)) } ) ,
\end{equation*}
where $\text{HM}(\{y_j\}_j)$ is the harmonic mean of $\{y_j\}_j \subset \mathbb{R}_+$. On the other hand, when $k = k_n^\star$, we have
\begin{equation*}
w_{i,n} = 0\ \forall\ i \notin c^{-1}(c(n)) , \quad k_i^\star = k_n^\star = k\ \forall\ i \in c^{-1}(c(n)) ,
\end{equation*}
so the summation in \eqref{eqSumOverI} is zero and the corresponding constraint is satisfied for any $x \in \mathbb{R}_+^{K \times N}$. Combining these observations, the constraint set in our special case simplifies to
\begin{equation*}
\mathcal{X} = \left\{ x \in \mathbb{R}_+^{K \times N} : \frac{ 2 }{ |c^{-1}(c(n))| \bar{\Delta}_{k,c(n)}^2 } \leq \text{HM} ( \{ x_{k,i} \}_{i \in c^{-1}(c(n)) } )\ \forall\ n \in [N] , k \neq k^{*}(c(n)) \right\} .
\end{equation*}
Now notice that these inequality constraints only depend on the bandit $c(n)$, not the individual agent $n$. Therefore, we further simplify the constraint set as follows:
\begin{equation*}
\mathcal{X} = \left\{ x \in \mathbb{R}_+^{K \times N} : \frac{ 2 }{ |c^{-1}(m)| \Delta_{m,k}^2 } \leq \text{HM} ( \{ x_{k,n} \}_{n \in c^{-1}(m) } )\ \forall\ m \in [M] , k \neq k^{*}(m) \right\} .
\end{equation*}
Next, consider the objective function in our special case. We rewrite it as
\begin{align}
f(x) & = \sum_{m=1}^M \sum_{k=1}^K \sum_{n=1}^N x_{k,n} \Delta_{k,n}' \mathbf{1} ( k_n^\star \neq k , c(n) = m ) = \sum_{m=1}^M \sum_{k=1}^K \sum_{n=1}^N x_{k,n} \Delta_{m,k} \mathbf{1} (k^{*}(m) \neq k , c(n) = m ) \\
& = \sum_{m=1}^M \sum_{k \neq k^{*}(m)} \Delta_{m,k} \sum_{n \in c^{-1}(m)} x_{k,n} = \sum_{m=1}^M \sum_{k \neq k^{*}(m)} \Delta_{m,k} |c^{-1}(m)| \text{AM} ( \{ x_{k,n} \}_{n \in c^{-1}(m)} ) ,
\end{align}
where $\text{AM}(\{y_j\}_j)$ denotes the arithmetic mean. Then for any $x \in \mathcal{X}$, we have
\begin{equation*}
f(x) \geq \sum_{m=1}^M \sum_{k \neq k^{*}(m)} \Delta_{m,k} |c^{-1}(m)| \text{HM} ( \{ x_{k,n} \}_{n \in c^{-1}(m)} ) \geq \sum_{m=1}^M \sum_{k \neq k^{*}(m)} \frac{2}{ \Delta_{m,k} } ,
\end{equation*}
where the first inequality is $\text{AM}(\{ y_j \}_j) \geq \text{HM}(\{ y_j \}_j)$ for any $\{ y_j \}_j \subset \mathbb{R}_+$ and the second holds since $x \in \mathcal{X}$. On the other hand, if we set $x_{k,n}^\star = 2 / ( |c^{-1}(c(n))|\Delta_{c(n),k}^2 )$, we clearly have
\begin{equation*}
\text{AM}(\{x_{k,n}^\star\}_{n \in c^{-1}(m)})  = \text{HM}(\{x_{k,n}^\star\}_{n \in c^{-1}(m)}) = 2 / ( |c^{-1}(m)| \Delta_{m,k}^2 )\ \forall\ m \in [M] .
\end{equation*}
Combined with the above, this shows that
\begin{equation*}
x^\star \in \mathcal{X} , \quad f(x^\star) = \sum_{m=1}^M \sum_{k \neq k^{*}(m)} \frac{2}{ \Delta_{m,k} } \leq f(x)\ \forall\ x \in \mathcal{X} ,
\end{equation*}
so $x^\star$ is optimal. Hence, using the lower bound from \cite{reda2022near}, we get
\begin{equation*}
\liminf_{T \rightarrow \infty} \frac{\mathrm{Reg}(T)}{ \log T } \geq \sum_{m=1}^M \sum_{k \neq k^{*}(m)} \frac{2}{\Delta_{m,k}}.
\end{equation*}

\subsection{Proof of Theorem \ref{thm:lowerbound2}}

The bulk of the proof involves showing that for any uniformly efficient policy,
\begin{equation} \label{eqLowerSingleAgent}
\liminf_{T \rightarrow \infty} \frac{\mathbb{E}[R_{T}^{(i)}]}{\log T} \geq 2 \Delta_m\ \forall\ i \in \cup_{m=1}^{M-1} \mathcal{I}_m .
\end{equation}
Assuming we can prove \eqref{eqLowerSingleAgent}, the first lower bound in the theorem statement follows easily, since
\begin{equation*}
\mathbb{E} [ \mathrm{Reg}(T) ] = \sum_{i=1}^N \mathbb{E}[R_{T}^{(i)}] \geq \sum_{m=1}^{M-1} \sum_{i \in \mathcal{I}_m } \mathbb{E}[R_{T}^{(i)}] ,
\end{equation*}
and the second follows from the first and the fact that, by assumption in Section \ref{sec:setup},
\begin{equation*}
2 \sum_{m=1}^{M-1} | \mathcal{I}_m | \Delta_m \geq 2 \Delta \sum_{m=1}^{M-1} | \mathcal{I}_m | = 2 \Delta ( N - | \mathcal{I}_M | ) \geq 2 \Delta N ( 1 - c_2 / M ) \geq \Delta N .
\end{equation*}
Hence, for the remainder of the proof, we fix $m \in [M-1]$ and $i \in \mathcal{I}_m$ and prove \eqref{eqLowerSingleAgent}.

Toward this end, we first reduce the real system to a fictitious system with a larger set of uniformly efficient policies. In the fictitious system, agent $i$ initially (i.e., at time $t=0$) observes the full sequence of rewards for all other agents, i.e., $\mathcal{X} := ( X_t^{(n)}(k) : t \in [T] , k \in [K] , n \in [N] \setminus \{ i \} )$, and thereafter does not communicate with other agents. Note that any policy in the real system can also be implemented in the fictitious system, since any information communicated by other agents is a function of $\mathcal{X}$. Hence, it suffices to prove \eqref{eqLowerSingleAgent} for any uniformly efficient policy in the fictitious system. 

To do so, we modify a standard lower bound approach that comprises Section 4.6, Lemma 15.1, and Theorem 16.2 of \cite{lattimore2020bandit}. For brevity, we focus on the modifications needed in our setting and refer the reader to the associated references for details.

To begin, we modify the probability measure construction from Section 4.6. First, we define a measurable space $(\Omega_T,\mathcal{F}_T)$, where $\Omega_T = ( [K] \times \mathbb{R} )^T \times \mathbb{R}^{ T \times K \times (N-1) }$, $\mathcal{F}_T = \mathcal{B}(\Omega_T)$, and $\mathcal{B}$ is the Borel $\sigma$-algebra. Next, let $\pi = ( \pi_t )_{t=1}^T$ be a uniformly efficient policy in the fictitious system. More precisely, each $\pi_t$ is a mapping from the history of arm pulls and rewards for agent $i$ (i.e., $I_1^{(i)} , X_1^{(i)}(I_1^{(i)}) , \ldots I_{t-1}^{(i)} , X_{t-1}^{(i)}(I_{t-1}^{(i)})$), along with the rewards of other agents (i.e., $\mathcal{X}$), to the distribution of the action $I_t^{(i)}$. Further, let $p_\nu$ be the density for a Gaussian random variable with mean $\nu$ and unit variance. Finally, let $\rho$ denote the counting measure and $\lambda$ any measure on $(\mathbb{R},\mathcal{B}(\mathbb{R}))$ for which all of the reward distributions  $\{ p_{\mu_{m,k}} \}_{(m,k) \in [M] \times [K]}$ are absolutely continuous with respect to $\lambda$. Then we define the probability measure
\begin{equation} \label{eqProbMeasure}
\mathbb{P}_{\mu,\mathcal{I},\pi}(B) = \int_B p_{\mu,\mathcal{I},\pi}(\omega) \left( (\rho \times \lambda)^T \times \lambda^{T \times K \times (N-1)} \right) d \omega\ \forall\ B \in \mathcal{F}_T ,
\end{equation}
where, for any $( ( j_t , x_t )_{t \in [T]} , ( x_{t,k,n} )_{t \in [T] , k \in [K] , n \in [N] \setminus \{i\}} ) \in \Omega_T$, $p_{\mu,\mathcal{I},\pi}$ is the density given by
\begin{align} \label{eqUglyDensity}
& p_{\mu,\mathcal{I},\pi} \left(  ( j_t , x_t )_{t \in [T]} , ( x_{t,k,n} )_{t \in [T] , k \in [K] , n \in [N] \setminus \{i\}} \right) \\
& \quad = \prod_{t \in [T]} \pi_t \left( j_t \middle| ( j_s, x_s )_{s \in [t-1]} , ( x_{t,k,n} )_{t \in [T] , k \in [K] , n \in [N] \setminus \{i\}} \right) p_{ \mu_{m, j_t } } ( x_t ) \prod_{ t \in [T] , k \in [K] , n \in [N] \setminus \{i\} } p_{ \mu_{c(n) , k} } ( x_{t,k,n} ) .
\end{align}

Next, we adapt the divergence decomposition from Lemma 15.1 to our setting. More precisely, we show it holds for \textit{certain} pairs of instances $(\mu,\mathcal{I})$ and $(\mu',\mathcal{I}')$. Specifically, let $\mu' = \mu$ and $\mathcal{I}' = \{ \mathcal{I}_j' \}_{j \in [M]}$, where
\begin{equation*}
\mathcal{I}'_j = \begin{cases} \mathcal{I}_j \setminus \{ i \} , & j = m \\ \mathcal{I}_j \cup \{ i \} , & j = m+1 \\ \mathcal{I}_j , & \text{otherwise} \end{cases} .
\end{equation*}
In words, the instance $(\mu',\mathcal{I}')$ is identical to $(\mu,\mathcal{I})$, except $i$ plays the $(m+1)^{\mathrm{th}}$ bandit in the former and the $m^{\mathrm{th}}$ in the latter. For simplicity, we thus write $\mathbb{P}_\mathcal{I} = \mathbb{P}_{\mu,\mathcal{I},\pi}$ and $\mathbb{P}_{\mathcal{I}'} = \mathbb{P}_{\mu',\mathcal{I}',\pi}$ for the probability measures \eqref{eqProbMeasure}, and $\mathbb{E}_\mathcal{I}$ and $\mathbb{E}_{\mathcal{I}'}$ for the associated expectations. We claim
\begin{equation} \label{eqDivergeDecomp}
D ( \mathbb{P}_{\mathcal{I}} , \mathbb{P}_{\mathcal{I}'} ) =  \sum_{k=1}^K \mathbb{E}_{\mathcal{I}} [ T_k^{(i)}(T) ] D ( p_{ \mu_{m,k}} , p_{ \mu_{m+1,k} } ) ,
\end{equation}
where $D$ denotes the KL divergence. The proof of \eqref{eqDivergeDecomp} follows the essentially the same logic as that of Lemma 15.1. The key difference is that our density \eqref{eqUglyDensity} includes the product term $\prod_{ t \in [T] , k \in [K] , n \in [N] \setminus \{i\} } p_{ \mu_{c(n) , k} } ( x_{t,k,n} )$, which arises from the reward sequences $\mathcal{X}$. However, since agents $n \neq i$ have the same reward distributions $p_{\mu_{c(n)},k}$ under both instances $(\mu,\mathcal{I})$ and $(\mu',\mathcal{I}')$, these product terms cancel in the proof, which yields \eqref{eqDivergeDecomp}.

Finally, we prove \eqref{eqLowerSingleAgent} using an approach similar to Theorem 16.2. We begin by upper bounding the summands on the right side of \eqref{eqDivergeDecomp}. For $k = k^*(m)$ (the optimal arm for $i$ in the instance $\mathcal{I}$), we have $\mu_{m,k^*(m)} = \mu_{m+1,k^*(m)}$ by assumption, which implies $D( p_{ \mu_{m,k^*(m)} } , p_{ \mu_{m+1,k^*(m)} } ) = 0$. For $k \neq k^*(m)$, we have
\begin{equation*}
D ( p_{ \mu_{m,k}} , p_{ \mu_{m+1,k} } ) = \frac{ ( \mu_{m,k} - \mu_{m+1,k} )^2 }{ 2 } \leq \frac{1}{2}  \leq \frac{\Delta_{m,k}}{ 2 \Delta_m } ,
\end{equation*}
where we used the fact that $p_\nu$ is Gaussian with mean $\nu$ and unit variance, the assumption $\mu_{m,k} \in [0,1]$, and the definition of $\Delta_m$, respectively. Combining arguments and using \eqref{eqDivergeDecomp}, we thus obtain
\begin{equation} \label{eqDivergeUpperBound}
D ( \mathbb{P}_{\mathcal{I}} , \mathbb{P}_{\mathcal{I}'} ) \leq \frac{1}{2 \Delta_m}  \sum_{k \in [K] \setminus \{k^*(m)\}}  \mathbb{E}_{\mathcal{I}} [ T_k^{(i)}(T) ] \Delta_{m,k} = \frac{ \mathbb{E}_{\mathcal{I}}  [ R_T^{(i)} ] }{ 2 \Delta_m }  .
\end{equation}
Next, we lower bound $D ( \mathbb{P}_{\mathcal{I}} , \mathbb{P}_{\mathcal{I}'} )$. Let $A = \{ T_{ k^*(m+1) }^{(i)}(T) > T/2 \}$ be the event that $i$ plays the best arm for the $(m+1)^{\mathrm{th}}$ bandit at least $T/2$ times. Then by assumption $k^*(m) \neq k^*(m+1)$, we know that
\begin{equation*}
\mathbb{E}_{\mathcal{I}} [ R_T^{(i)} ] \geq \frac{T \Delta_{ m, k^*(m+1) } }{2} \mathbb{P}_{ \mathcal{I} } ( A ) \geq \frac{ T \Delta_m }{ 2 } \mathbb{P}_{ \mathcal{I} } ( A ) \geq \frac{ T \Delta }{ 2 } \mathbb{P}_{ \mathcal{I} } ( A ) ,
\end{equation*}
and by definition of $\mathcal{I}'$ (where $i$ plays the $(m+1)^{\mathrm{th}}$ bandit), we similarly have $\mathbb{E}_{\mathcal{I}'} [ R_T^{(i)} ] \geq T \Delta \mathbb{P}_{ \mathcal{I}' } ( A^C ) / 2$. Combining these inequalities and using Theorem 14.2 of \cite{lattimore2020bandit}, we get
\begin{equation} \label{eqDivergeLowerBound}
\mathbb{E}_{\mathcal{I}} [ R_T^{(i)} ] + \mathbb{E}_{\mathcal{I}'} [ R_T^{(i)} ] \geq \frac{T \Delta}{2} \left( \mathbb{P}_{ \mathcal{I} } ( A ) + \mathbb{P}_{ \mathcal{I}' } ( A^C ) \right) \geq \frac{T \Delta}{4} \exp \left( - D ( \mathbb{P}_{\mathcal{I}} , \mathbb{P}_{\mathcal{I}'} ) \right) .
\end{equation}
Therefore, combining \eqref{eqDivergeUpperBound} and \eqref{eqDivergeLowerBound} and letting $T \rightarrow \infty$, we obtain
\begin{equation*}
\liminf_{T \rightarrow \infty} \frac{ \mathbb{E}_{\mathcal{I}} [ R_T^{(i)} ] }{ \log T } \geq 2 \Delta_m \liminf_{T \rightarrow \infty}  \frac{ D ( \mathbb{P}_{\mathcal{I}} , \mathbb{P}_{\mathcal{I}'} ) }{ \log T } \geq 2 \Delta_m \liminf_{T \rightarrow \infty}  \left( 1 - \frac{ \log ( 4 (\mathbb{E}_{\mathcal{I}} [ R_T^{(i)} ] + \mathbb{E}_{\mathcal{I}'} [ R_T^{(i)} ]) / \Delta )}{ \log T } \right)  = 2 \Delta_m ,
\end{equation*}
where the equality holds by the uniformly efficient assumption (which states that the group regret on any problem instance, and hence the individual regrets $\mathbb{E}_{\mathcal{I}} [ R_T^{(i)} ]$ and $\mathbb{E}_{\mathcal{I}'} [ R_T^{(i)} ]$, are $o(T^\gamma)$ for arbitrarily small $\gamma > 0$).
\end{appendices}

\end{document}